\documentclass[1p, 10pt]{elsarticle}

\makeatletter
\def\ps@pprintTitle{%
 \let\@oddhead\@empty
 \let\@evenhead\@empty
 \def\@oddfoot{}%
 \let\@evenfoot\@oddfoot}
\makeatother

\usepackage[latin1]{inputenc}
\usepackage{amsmath,bm}
\usepackage{amsfonts}
\usepackage{color}
\usepackage{amssymb}
\usepackage{graphicx,subfigure,caption}
\usepackage{algpseudocode, algorithm}
\usepackage{setspace}
\usepackage{lipsum}
\usepackage{cite}
\usepackage{array,multirow}
\usepackage{stfloats}
\usepackage{array,arydshln}

\newcommand{\bxi}{{\bf x_i}}
\newcommand{\byi}{{\bf y_i}}
\newcommand{\by}{{\bf y}}

\newcommand{\bb}{{\bf b}}

\newcommand{\bV}{{\bf V}}
\newcommand{\bU}{{\bf U}}

\newcommand{\bX}{{\bf X}}
\newcommand{\bI}{{\bf I}}
\newcommand{\bY}{{\bf Y}}
\newcommand{\bA}{{\bf A}}
\newcommand{\bB}{{\bf B}}
\newcommand{\bC}{{\bf C}}
\newcommand{\bD}{{\bf D}}
\newcommand{\bZ}{{\bf Z}}

\newcommand{\bW}{{\bf W}}
\newcommand{\bS}{{\bf S}}

\newcommand{\bE}{{\bf E}}
\newcommand{\bH}{{\bf H}}

\newcommand{\bei}{{\boldsymbol \epsilon_i}}
\newcommand{\bsig}{{\boldsymbol \Sigma}}

\usepackage{amsthm}

\newtheoremstyle{exa}
  {\topsep} 
  {0 in} 
  {\itshape} 
  {} 
  {\bfseries} 
  {.} 
  {.5em} 
  {} 

\theoremstyle{definition}
\newtheorem{mydef}{Definition}
\theoremstyle{exa}
\newtheorem{mytheo}{Theorem}
\theoremstyle{exa}
\newtheorem{myprop}{Proposition}
\theoremstyle{plain}
\newtheorem{mylemma}{Lemma}

\DeclareMathOperator*{\argmin}{\arg\!\min}

\newcommand{\spcell}[2][c]{%
  \begin{tabular}[#1]{@{}c@{}}#2\end{tabular}}

\begin{document}
\begin{frontmatter}
\author{Milad Kharratzadeh, Mark Coates \\ Electrical and Computer Engineering Department, McGill University \\  Montreal, Quebec, Canada \\ milad.kharratzadeh@mail.mcgill.ca, mark.coates@mcgill.ca}

\title{Sparse Multivariate Factor Regression}
\begin{abstract}
We consider the problem of multivariate regression in a setting where the relevant predictors could be shared among different responses. We propose an algorithm which decomposes the coefficient matrix into the product of a long matrix and a wide matrix, with an elastic net penalty on the former and an $\ell_1$ penalty on the latter. The first matrix linearly transforms the predictors to a set of latent factors, and the second one regresses the responses on these factors. Our algorithm simultaneously performs dimension reduction and coefficient estimation and automatically estimates the number of latent factors from the data. Our formulation results in a non-convex optimization problem, which despite its flexibility to impose effective low-dimensional structure, is difficult, or even impossible, to solve exactly in a reasonable time. We specify an optimization algorithm based on alternating minimization with three different sets of updates to solve this non-convex problem and provide theoretical results on its convergence and optimality. Finally, we demonstrate the effectiveness of our algorithm via experiments on simulated and real data.
\end{abstract}
\end{frontmatter}
\section{Introduction}
Multivariate regression analysis, also known as multiple--output regression, is concerned with modelling the relationships between a set of real--valued output vectors, known as responses, and a set of real--valued input vectors, known as predictors or features. The multivariate responses are measured over the same set of predictors and are often correlated. Hence, the goal of multivariate regression is to exploit these dependencies to learn a predictive model of responses based on an observed set of input vectors paired with corresponding outputs. Multiple--output regression can also be seen as an instance of the problem of multi--task learning, where each task is defined as predicting individual responses based on the same set of predictors. The multivariate regression problem is encountered in numerous fields including finance~\citep{Lee10}, computational biology~\citep{Mon99}, geostatistics~\citep{Ran08}, chemometrics~\citep{Wold01}, and neuroscience~\citep{Har03}.

In this paper, we are interested in multivariate regression tasks where it is reasonable to believe that the responses are related to factors, each of which is a sparse linear combination of the predictors. Our model further assumes that the relationships between the factors and the responses are sparse. This type of structure occurs in a number of applications and we provide two examples in later sections. 

Given $p$--dimensional predictors $\bxi = (x_{i1}, \ldots, x_{ip})^T \in \mathbb{R}^p$ and $q$--dimensional responses $\byi = (y_{i1}, \ldots, y_{iq})^T \in \mathbb{R}^q$ for the $i$-th sample, we assume there is a linear relationship between the inputs and outputs as follows:
\begin{equation} \label{eq:mv}
\byi = \bD^T \bxi + \bei, \qquad i=1, \ldots, N,
\end{equation}
where $\bD_{p\times q}$ is the regression coefficient matrix and $\bei = (\epsilon_{i1}, \ldots, \epsilon_{iq})$ is the vector of errors for the $i$-th sample. We can combine these $N$ equations into a single matrix formula:
\begin{equation} \label{eq:mv2}
\bY = \bX \bD + \bE,
\end{equation}
where $\bX$ denotes the $n\times p$ matrix of predictors with $\bxi^T$ as its $i$-th row,  $\bY$ denotes the $n\times q$ matrix of responses with $\byi^T$ as its $i$-th row, and  $\bE$ denotes the $n\times q$ matrix of errors with $\bei^T$ as its $i$-th row. For $q=1$, this multivariate linear regression model reduces to the well--known, univariate linear regression model.

 We assume that the columns of $\bX$ and $\bY$ are centred and hence the intercept terms are omitted and the columns of $\bX$ are normalized. We also assume that the error vectors for $N$ samples are iid Gaussian random vectors with zero mean and covariance $\bsig$, i.e. $\bei \sim \mathcal{N}({\bf 0}, \bsig), i=1,\ldots,N$.
 
In the absence of additional structure, many standard procedures for solving (\ref{eq:mv2}), such as linear regression and principal component analysis, are not consistent unless $p/n\to 0$. Thus, in a high-dimensional setting where $p$ is comparable to or greater than $n$, we need to  impose some low-dimensional structure on the coefficient matrix. For instance, element-wise sparsity can be imposed by constraining the $\ell_1$ norm of the coefficient matrix, $\| \bD\|_{1,1}$~\citep{Tib11, Wai09}. This regularization is equivalent to solving $q$ separate univariate lasso regressions for every response; thus we consider tasks separately. Another way to introduce sparsity is to consider the mixed $\ell_1/\ell_{\gamma}$ norms ($\gamma> 1$). In this approach (sometimes called group lasso), the mixed norms impose a block-sparse structure where each row is either all zero or mostly zeros. Particular examples, among many other works, include results using the $\ell_1/\ell_{\infty}$ norm~\citep{Tur05, Zha08}, and the $\ell_1/\ell_2$ norm~~\citep{Yua06, Obo11}. Also, there are the so-called ``dirty'' models which are superpositions of simpler low-dimensional structures such as element-wise and row-wise sparsity~\citep{Jal10, pen10} or sparsity and low rank~\citep{Chen11, Cha11}.  

Another approach is to impose a constraint on the rank of the coefficient matrix. In this approach, instead of constraining the regression coefficients directly, we can apply penalty functions on the rank of $\bD$, its singular values and/or its singular vectors~\citep{pou13,Chen12,Chun10,Yua07,Rein98}. These algorithms belong to a broad family of dimension-reduction methods known as {\it linear factor regression}, where the responses are regressed on a set of {\it factors} achieved by a linear transformation of the predictors. The coefficient matrix is decomposed into two matrices: $\bD = \bA_{p\times m}\bB_{m\times q}$. Matrix $\bA$ transforms the predictors into $m$ latent factors, and matrix $\bB$ determines the factor loadings. 

\subsection*{Our contributions}
Here, we propose a novel algorithm which performs sparse multivariate factor regression (SMFR). We jointly estimate matrices $\bA$ and $\bB$ by minimizing the mean-squared error, $\|\bY-\bX\bA\bB\|_F^2$, with an elastic net penalty on $\bA$ (which promotes grouping of correlated predictors and the interpretability of the factors)  and an $\ell_1$ penalty on $\bB$ (which enhances the accuracy and interpretability of the predictions). We provide a formulation to estimate the  number of effective latent factors, $m$. To the best of our knowledge, our work is the first to strive for low-dimensional structure by imposing sparsity on both factoring and loading matrices as well as the grouping of the correlated predictors. This can result in a set of interpretable factors and loadings with high predictive power; however, these benefits come at the cost of a non-convex objective function. Most current approaches for multivariate regression  solve a convex problem (either through direct formulation or by relaxation of a non-convex problem) to impose low-dimensional structures on the coefficient matrix. Although non-convex formulations, such as the one introduced here, can be employed to achieve very effective representations in the context of multivariate regression, there are few theoretical performance guarantees for optimization schemes solving such problems. We formulate our problem in Section~\ref{sec:setup}. In Section~\ref{sec:theory}, we propose an alternating minimization scheme with three sets of updates to solve our problem and provide theoretical guarantees for its convergence and optimality. We show that under mild conditions on the predictor matrix, every limit point of the minimization algorithm is a stationary point of the objective function and if the starting point is close enough to a local or global minimum, our algorithm converges to that point. Through analysis of simulations on synthetic datasets in Section~\ref{sec:synth} and two real-world datasets in Section~\ref{sec:real}, we show that compared to other multivariate regression algorithms, our proposed algorithm can provide a more effective representation of the data, resulting in a higher predictive power.

\subsection*{Related Methods}\label{sec:rel}
Many multivariate regression techniques impose a low-dimensional structure on the coefficient matrix. Element-wise sparsity, here noted as LASSO, is the most common approach where the cost function is defined as $\|\bD\|_{1,1}$~\citep{Tib11, Wai09}. An extension of LASSO to the multivariate case is the row-wise sparsity with the $\ell_1/\ell_2$ norm as the cost function: $\|\bD\|_{1,2}$~\citep{Yua06, Obo11}. Peng et al. proposed a method, called RemMap~\citep{pen10}, which imposes both element-wise and row-wise sparsity and solves the following problem: 
\begin{equation*}
\min_{\bD} \|\bY-\bX\bD\|_F^2 + \lambda_1 \|\bD\|_{1,1} + \lambda_2 \| \bD\|_{1,2}.
\end{equation*}
In an alternative approach, \citep{Chun10} extended the partial least squares (PLS) framework by imposing an additional sparsity constraint and proposed Sparse PLS (SPLS). 

Another common idea is to employ dimensionality reduction techniques to find the underlying latent structure in the data. One of the most basic algorithms in this class is an approach called Reduced Rank Regression (RRR)~\citep{Velu98} where the sum-of-squares error is minimized under the constraint that rank$(\bD) \leq r$ for some $r\leq \min\{p,q\}$. It is easy to show that one can find a closed-form solution for $\bD$ based on the singular value decomposition of $\bY^T\bX(\bX^T\bX)^{-1}\bX^T\bY$. However, similar to least-squares, the solution of this problem without appropriate regularization exhibits poor predictive performance and is not suitable for high-dimensional settings. 
Another popular approach is to use the trace norm as the penalty function:
\begin{equation}
\min_{\bD} \|\bY-\bX\bD\|_F^2 + \lambda \sum_{j=1}^{\min\{p,q\}} \sigma_j(\bD),
\end{equation}
where $\sigma_j(\bD)$ denotes the $j$'th singular value of $\bD$. The trace norm regularization has been extensively studied in the literature~\citep{Yua07, Arg08, Ji09, Zha12}. It imposes sparsity in the singular values of $\bD$ and therefore, results in a low-dimensional solution (higher values of $\lambda$ correspond to achieving solutions of lower rank).

Many papers study problems of the following form:
\begin{equation}
\min_{\bD} \|\bY-\bX\bD\|_F^2 + g(\bD) \ \textrm{  s.t.  \ rank}(\bD)\leq r\leq \min\{p,q\},
\end{equation}
where $g(\bD)$ is a regularization function over $\bD$. For instance, in~\citep{Mukh11}, a ridge penalty is proposed with $g(\bD) = \lambda\|\bD\|_F^2$. Often it is assumed that $\bD = \bA_{p\times r} \bB_{r\times q}$ (and thus rank$(\bA)\leq r$, rank$(\bB)\leq r$, and consequently rank$(\bD)\leq r$) and the problem is formulated in terms of $\bA$ and $\bB$. In~\citep{Kum12}, $g(\bA,\bB)=\lambda_1\|\bA\|_{1,1} + \lambda_2 \|\bB\|_F^2$.  An algorithm called Sparse Reduced Rank Regression (SRRR) is proposed in \citep{Chen12} and further studied in~\citep{Ma14}, where $g(\bA,\bB)=\lambda\|\bA\|_{1,2}$ with an extra constraint that $\bB\bB^T = {\bf I}$. In~\citep{Arg08}, $g(\bA,\bB) = \lambda\|\bB\|_{2,1}^2$ with an extra constraint that $\bA^T\bA = {\bf I}$, and it is assumed that $p\leq q$ and $r=p$. Dimension reduction is achieved by the constraint on $\bB$ which forces many rows to be zero which consequently cancels the effects of the corresponding columns in $\bA$.

%

Our problem formulation differs in three important ways: (i) sparsity constraints are imposed on {\it both} $\bA$ and $\bB$; (ii) the elastic net penalty enables us to control the level of sparsity for each matrix separately and also provides the grouping of correlated predictors; and (iii) the number of factors is determined directly, without the need for cross-validation. We will discuss the second and third aspects in detail in the next section. The first difference has substantial consequences; when decomposing the coefficient matrix into two matrices, the first matrix has the role of aggregating the input signals to form the latent factors and the second matrix performs a multivariate regression on these factors. Imposing sparsity on $\bA$ enhances the variable selection as well as the interpretability of the achieved factors. Also, as originally motivated by LASSO, we would like to impose the sparsity constraint on $\bB$ in order to improve the interpretability and prediction performance. 

Imposing sparsity on both $\bA$ and $\bB$ means that for our model to make sense for a specific problem, the outputs should be related in a sparse way to a common set of factors that are derived as a sparse combination of the inputs. For example, in analyzing the S\&P 500\footnote{Standard and Poors index of 500 large-cap US equities http://ca.spindices.com/indices/equity/sp-500} stocks, it is well-known that returns exhibit much stronger correlations for companies that belong to the same industry sector. The memberships of each sector are not always so clear, because a company may have several diverse activities that generate revenue. So if we are using just concurrent stock returns to try to predict those of other companies, it is reasonable to assume that factors representing industry sectors should appear~\citep{Chou12}. Since we do not expect the main sectors to overlap much, a company will not be present in many factors; so, it is reasonable to assume that $\bA$ is sparse. Moreover, most companies will only be predicted by one or two such factors, so it makes sense that $\bB$ is also sparse.


%

There is a connection between Sparse PCA~\citep{Zou06s} and our algorithm, in the case where $\bY$ is replaced by $\bX$ (i.e., $\bX$ is regressed on itself). We investigate this relationship in Section~\ref{SPCA}.  

\section{Problem Setup}\label{sec:setup}


 In this work, we introduce a novel low-dimensional structure where we decompose $\bD$ into the product of two sparse matrices $\bA_{p\times m}$ and $\bB_{m\times q}$ where $m < \min (p,q)$. 
 This decomposition can be interpreted as first identifying a set of $m$ factors which are derived by some linear transformation of the predictors (through matrix $\bA$) and then identifying the  transformed regression coefficient matrix $\bB$ to estimate the responses from these $m$ factors. We provide a framework to find $m$, the  number of effective latent factors, as well as the transforming and regression matrices, $\bA$ and $\bB$. For a fixed $m$, define:
\begin{equation}
\widehat{\bA}_m, \widehat{\bB}_m = \displaystyle \argmin_{\bA_{p\times m}, \bB_{m\times q}}  f(\bA, \bB), \label{eq:sub2}
\end{equation} 
where
\begin{equation}
f(\bA,\bB) \!=\! \frac{1}{2} \| \bY\!-\!\bX\bA\bB\|_F^2 \!+\! \lambda_1 \|\bA\|_{1,1} \!+\! \lambda_2 \|\bB\|_{1,1}\!+\! \lambda_3 \|\bA\|_{F}^2.  \label{eq:sub3} 
\end{equation}
Then, we solve the following optimization problem:
\begin{equation}\label{eq:sub1} 
 \widehat{m} \!= \!\max (m)\leq r   \text{ \ s.t. \  rank$(\widehat{\bA}_m) \!=\! $ rank}(\widehat{\bB}_m) \!=\! m, 
\end{equation}
where $r$ is a problem-specific bound on the number of factors. We then choose $\widehat{\bA}_{\widehat{m}}$ and $\widehat{\bB}_{\widehat{m}}$ as solutions. Thus, we find the maximum number of factors such that the solution of (\ref{eq:sub2}) has full rank factor and loading matrices. In other words, we find the maximum $m$ such that the best possible regularized reconstruction of responses, i.e., the solution of (\ref{eq:sub2}), results in a model where the factors (columns of of $\widehat{\bA}$) and their contributions to the responses (rows of $\widehat{\bB}$) are linearly independent. To achieve this, we initialize $\bA$ to have $r$ columns, $\bB$ to have $r$ rows, and set $m=r$, solve the problem (\ref{eq:sub2})--(\ref{eq:sub3}), check for the full rank condition; if not satisfied, set $m=m-1$, and repeat the process until we identify an $m$ that satisfies the rank condition.

In the remainder of this section, we first discuss the desirable properties resulting from the choice of our penalty function, and then explain the reasoning behind how we estimate $m$.

\subsection{Controlling Sparsity Levels in Matrices $\bA$ and $\bB$ Separately}
Consider the case where $\lambda_3=0$. We have:
\begin{align}
\min_{\bA,\bB} f(\bA,\bB) & = \min_{\bA,\bB} \frac{1}{2} \| \bY-\bX\bA\bB\|_F^2 + \lambda_1 \|\bA\|_{1,1} + \lambda_2 \|\bB\|_{1,1} \label{eq:first_opt} \\
& = \min_{\bA',\bB', c} \frac{1}{2} \| \bY-\bX\bA'\bB'\|_F^2 + \frac{\lambda_1}{c} \|\bA'\|_{1,1} + \lambda_2c \|\bB'\|_{1,1} \label{eq:mid_opt}\\
& = \min_{\bA',\bB'} \frac{1}{2} \| \bY-\bX\bA'\bB'\|_F^2 + \sqrt{\lambda_1\lambda_2\|\bA'\|_{1,1}\|\bB'\|_{1,1}}, \label{eq:last_opt}
\end{align}
where $0\leq c$, $\bA = \bA'/c$ and $\bB' = c\bB$. In problem (\ref{eq:last_opt}), we do not have any separate control over the sparsity levels of $\bA'$ and $\bB'$. If $(\bA',\bB')$ is a solution to the last problem, then there is a $c^*$ for which $(\bA'/c^*, c^*\bB)$ is a solution to the first problem. Therefore, we cannot control the sparsity levels of the two matrices by just including two $\ell_{1}$ norms. Incorporating the elastic net penalty for $\bA$ resolves this issue immediately since the equivalence between optimization problems (\ref{eq:first_opt}) and (\ref{eq:last_opt}) does not hold any more.

\subsection{Grouping of Correlated Features}

In this section, we show the $i$'th row of a matrix $\bX$ by $\bX_{i\cdot}$ and its $j$'th column by $\bX_{\cdot j}$. Remember that matrix $\bA$ has the role of combining the relevant features to form the latent factors which will be used later in the second layer by matrix $\bB$ for estimating the outputs. If there are two highly correlated features we expect them to be grouped together in forming the factors. In other words, we expect them to be both present in a factor or both absent. Inspired by Theorem 1 in the original paper of Zou and Hastie on elastic net~\citep{Zou05enet}, we prove in this section that elastic net penalty enforces the grouping of correlated features in forming the factors. 

The columns of $\bX$ correspond to different features. We assume that all columns of $\bX$ are centred and normalized. Thus, the correlation between the $i$'th and the $j$'th features is $\rho_{ij} \triangleq \bX_{\cdot i}^T \bX_{\cdot j}$.

\begin{mylemma}
Consider solving the following problem for given $\lambda_1$ and $\lambda_3$:
\begin{align}
\widehat{\bA} = \argmin_{\bA} f(\bA,\bB) = \argmin_{\bA} \frac{1}{2} \| \bY-\bX\bA\bB\|_F^2 + \lambda_1 \|\bA\|_{1,1} + \lambda_3 \|\bA\|_F^2.
\end{align}
Then, if $\widehat{\bA}_{ik}\widehat{\bA}_{jk}>0$, we have:
\begin{align}
\frac{2\lambda_3}{\|\bY\|_F\|\bB_{k\cdot}\|_F} |\widehat{\bA}_{ik}-\widehat{\bA}_{jk}| \leq \sqrt{2(1-\rho_{ij})}
\end{align}
\end{mylemma}
This lemma says, for instance, that if the correlation between features $i$ and $j$ is really high (i.e., $\rho_{ij} \approx 1$), then the difference between their corresponding weights in forming the $k$'th factor, $|\widehat{\bA}_{ik}-\widehat{\bA}_{jk}|$, would be very close to 0. If $\bX_{\cdot i}$ and $\bX_{\cdot j}$ are negatively correlated, we can state the same lemma for $\bX_{\cdot i}$ and $-\bX_{\cdot j}$ and use $|\rho_{ij}|$.\\

\begin{proof}
The condition $\widehat{\bA}_{ik}\widehat{\bA}_{jk}>0$ means that $\widehat{\bA}_{ik}\neq 0$, $\widehat{\bA}_{jk}\neq 0$, and $\text{sign}(\widehat{\bA}_{ik}) = \text{sign}(\widehat{\bA}_{jk})$. So, we have:
\begin{align}
\frac{\partial f}{\partial \bA_{ik}}\Bigr|_{\bA_{ik} = \widehat{\bA}_{ik}} & = (-\bX^T\bY\bB^T)_{ik} +  (\bX^T\bX\widehat{\bA}\bB\bB^T)_{ik} +  \lambda_1 \text{sign}(\widehat{\bA}_{ik}) + 2\lambda_3 \widehat{\bA}_{ik}  \\ 
& = -(\bX_{\cdot i})^T (\bY\bB^T)_{\cdot k} + (\bX_{\cdot i})^T (\bX\widehat{\bA}\bB\bB^T)_{\cdot k} +  \lambda_1 \text{sign}(\widehat{\bA}_{ik}) + 2\lambda_3 \widehat{\bA}_{ik} \\
\frac{\partial f}{\partial \bA_{jk}}\Bigr|_{\bA_{jk}=\widehat{\bA}_{jk}} & = (-\bX^T\bY\bB^T)_{jk} +  (\bX^T\bX\widehat{\bA}\bB\bB^T)_{jk} +  \lambda_1 \text{sign}(\widehat{\bA}_{jk}) + 2\lambda_3 \widehat{\bA}_{jk} \\
& = -(\bX_{\cdot j})^T (\bY\bB^T)_{\cdot k} + (\bX_{\cdot j})^T (\bX\widehat{\bA}\bB\bB^T)_{\cdot k} +  \lambda_1 \text{sign}(\widehat{\bA}_{jk}) + 2\lambda_3 \widehat{\bA}_{jk}
\end{align}
Due to the optimality of $\widehat{\bA}$, both derivatives are equal to zero. Equating the two equations, we get:
\begin{align}
2\lambda_3 (\widehat{\bA}_{ik} - \widehat{\bA}_{jk}) = (\bX_{\cdot i} - \bX_{\cdot j})^T (\bY - \bX\widehat{\bA}\bB) (\bB_{k \cdot})^T
\end{align}
Thus, 
\begin{align}
2\lambda_3 |\widehat{\bA}_{ik} - \widehat{\bA}_{jk}| & = |(\bX_{\cdot i} - \bX_{\cdot j})^T (\bY - \bX\widehat{\bA}\bB) (\bB_{k \cdot})^T| \\
& \leq \| \bX_{\cdot i} - \bX_{\cdot j} \|_F  \|\bY - \bX\widehat{\bA}\bB\|_F \|\bB_{k \cdot}\|_F
\end{align}
Since the columns of $\bX$ are normalized, we have $\| \bX_{\cdot i} - \bX_{\cdot j} \|_F = \sqrt{2(1-\rho_{ij})}$. Also, by definition, we have:
\begin{align}
f(\widehat{\bA},\bB) \leq f({\bf 0},\bB) \Rightarrow \|\bY - \bX\widehat{\bA}\bB\|^2_F + \lambda_1 \|\widehat{\bA}\|_{1,1} + \lambda_3 \|\widehat{\bA}\|_{1,1} \leq \|\bY\|_F^2 \Rightarrow \|\bY - \bX\widehat{\bA}\bB\|_F \leq \|\bY\|_F
\end{align}
Combining all these together gives the desired result. 
\end{proof}

\subsection{Estimating the number of effective factors}
In choosing $m$, we want to avoid both overfitting (large $m$) and lack of sufficient learning power (small $m$). In general, we only require $m\leq\min(p,q)$; however, in practical settings where $p$ and $q$ are very large, we impose an upper bound on $m$ to have a reasonable number of factors and avoid overfitting. This upper bound, denoted by $r$, is problem-specific and should be chosen by the programmer. For instance, continuing our example of analysing the S\&P 500 stocks, most economists identify 10-15 primary financial sectors, so a choice of $r=15$ or $20$ is reasonable since we expect the factors representing industries to appear, but we allow the algorithm to find the optimal $m$ from the data. In order to have the maximum learning power, we find the maximum $m\leq r$ for which the solutions satisfy our rank conditions. This motivates starting with $m=r$ and decreasing it until the conditions hold (as opposed to starting with $m=1$ and increasing the dimension).

The full rank conditions are employed to guarantee a good estimate of the number of ``effective" factors. An effective factor explains some aspect of the response data but cannot be constructed as a linear combination of other factors (otherwise it is superfluous). We therefore require the estimated factors to be linearly independent. In addition, we require that the rows of $\bB$, which determine how the factors affect the responses, are linearly independent. If we do not have this latter independence, we could reduce the number of factors and still obtain the same relationship matrix $\bD$, so at least one of the factors is superfluous. By enforcing that $\bA$ and $\bB$ are full rank, we make sure that the estimated factors are linearly independent in both senses, and thus $\widehat{m}$ is a good estimate of the number of effective factors. 

In estimating the number of factors, we differ fundamentally from the common approach in literature. Setting aside the differences in the choice of regularization, most algorithms minimize a cost function as in (\ref{eq:sub2}) for a fixed $m$ {\it without} the rank condition in (\ref{eq:sub1}). They then use cross-validation to find the optimal $\widetilde{m}$~\citep{Yua07, Rein98, Chen12,pen10}. The criterion in choosing $\widetilde{m}$ is thus the cross-validation error. In contrast, we strive to find the largest $\widehat{m}$ such that the solutions of (\ref{eq:sub2}) have full rank and the estimated factors are linearly independent. Finding  $\widetilde{m}$ via cross-validation may result in non-full rank solutions with linearly dependent factors. Therefore, some factors can be expressed as linear combinations of other factors and can be viewed as redundant. Including redundant factors (via a non-full rank $\bA$) could help to improve the sparsity of $\bB$, but our goal is to have the minimum necessary set of factors, not to have a sparse $\bB$ at any cost. Later, with experiments on synthetic and real data, we show that our approach towards choosing the number of factors results in better predictive performance as well as more interpretable factors compared to other techniques that apply cross-validation.

\section{Optimization Technique and Theoretical Results} \label{sec:theory}
The optimization problem defined in~(\ref{eq:sub2}--\ref{eq:sub1}) is not a convex problem  and it is difficult, if not impossible, to solve exactly (i.e., to find the global optimum) in polynomial time. Therefore, we have to employ heuristic algorithms, which may or may not converge to  a stationary solution~\citep{Pow73}. In this section, we propose an alternating minimization algorithm with three different sets of updates, and provide theoretical results for each of them. 

\subsection{Optimization}
For a fixed $m$, the objective function in (\ref{eq:sub2}) is biconvex in $\bA$ and $\bB$; it is not convex in general, but is convex if either $\bA$ or $\bB$ is fixed. Let us define $\bC = (\bA, \bB)$. To solve (\ref{eq:sub2}) for a fixed $m$, we perform Algorithm 1, with an arbitrary, non-zero starting value~$\bC_0 = (\bA_0, \bB_0)$ (see Section~\ref{sec:synth} for a discussion on the choice of the starting point).
The stopping criterion is related to the convergence of the value of function $f$, not the convergence of its arguments. In our experiments, we assume $f$ has converged, if $\frac{|f_i - f_{i+1}|}{f_i} < \epsilon$ where the default value of the tolerance parameter, $\epsilon$, is $1E-5$; i.e., the algorithm stops if the relative changes in $f$ are less than $0.001\%$.

\begin{algorithm}[!hb]
\caption{Solving problem (\ref{eq:sub2}) for fixed $m$}
\begin{spacing}{1.1}
\begin{algorithmic}
\State $\bA \gets \bA_0$, $\bB \gets \bB_0$, $i \gets 0$ 
\While {stopping criterion not satisfied} \vspace*{-.2 in}
\State \begin{equation} \bB_{i+1} \gets \text{update $\bB$ with $\bA$ fixed at $\bA_i$} \end{equation} 
\State \vspace*{-.2 in}\begin{equation} \bA_{i+1}\! \gets \text{update $\bA$ with $\bB$ fixed at $\bB_{i+1}$} \end{equation} 
\State $i \gets i+1$ 
\EndWhile
\State $\widehat{\bA}, \widehat{\bB} \gets $ \ values of $\bA$ and $\bB$ at convergence
\end{algorithmic}
\end{spacing}
\end{algorithm}

We consider three different types of updates:
\begin{itemize}
\item Basic updates:
\begin{align}
\bB_{i+1} & \leftarrow  \argmin_B \frac{1}{2} \| \bY-\bX\bA_i\bB\|_F^2 + \lambda_2 \|\bB\|_{1,1} \label{eq:opt1}\\
\bA_{i+1} & \leftarrow  \argmin_A \frac{1}{2} \| \bY-\bX\bA\bB_{i+1}\|_F^2 + \lambda_3 \|\bA\|_F^2  + \lambda_1 \|\bA\|_{1,1} \label{eq:opt2}
\end{align}
\item Proximal updates:
\begin{align}
\bB_{i+1} & \leftarrow  \argmin_B \frac{1}{2} \| \bY-\bX\bA_i\bB\|_F^2 + \lambda_2 \|\bB\|_{1,1} + \beta_i \|\bB-\bB_{i}\|_F^2 \\
\bA_{i+1} & \leftarrow  \argmin_A \frac{1}{2} \| \bY-\bX\bA\bB_{i+1}\|_F^2 + \lambda_3 \|\bA\|_F^2  + \lambda_1 \|\bA\|_{1,1} + \alpha_i \|\bA-\bA_i\|_F^2
\end{align}
where $\alpha_{min}\leq\alpha_i\leq\alpha_{max}$ and $\beta_{min}\leq\beta_i\leq\beta_{max}$; i.e., they are bounded from both sides.
\item Prox-linear updates: 
\begin{align}
\bB_{i+1} & \leftarrow S_{\lambda_2/\beta_i}(\widetilde{\bB}_i - g(\bA_i, \widetilde{\bB}_i) / \beta_i) \\
\bA_{i+1} & \leftarrow S_{\lambda_1/\alpha_i}(\widetilde{\bA}_i - h(\widetilde{\bA}_i,\bB_{i+1}) / \alpha_i)
\end{align}
where $S$ is the soft-thresholding function, $S_{\tau}(\nu) = \text{sign}(\nu) \max(|\nu|-\tau,0)$, and:
\begin{align}
g(\bA,\bB) & = \frac{\partial}{\partial \bB} \left(\frac{1}{2} \| \bY-\bX\bA\bB\|_F^2\right) = -\bA^T\bX^T\bY + \bA^T\bX^T\bX\bA\bB \\
h(\bA,\bB) & = \frac{\partial}{\partial \bA} \left(\frac{1}{2} \| \bY-\bX\bA\bB\|_F^2 + \lambda_3\|\bA\|_F^2\right) = -\bX^T\bY\bB^T+\bX^T\bX\bA\bB\bB^T + 2\lambda_3\bA
\end{align}
are the derivatives of $f(\bA,\bB)$ without the $\ell_1$ penalties. Also, $\alpha_i$ and $\beta_i$ are multipliers that have to be greater than or equal to the Lipschitz constants of $h(\bA,\bB_{i+1})$ and $g(\bA_i,\bB)$ respectively. Since: 
\begin{align}
\|g(\bA_i,\bB) - g(\bA_i,\bB')\|_F & = \|\bA_i^T\bX^T\bX\bA_i(\bB-\bB')\|_F \leq \|\bA_i^T\bX^T\bX\bA_i\|_F\|\bB-\bB'\|_F \label{eq:Lip1}\\
\|h(\bA,\bB_{i+1}) - h(\bA',\bB_{i+1})\|_F & = \|\bX^T\bX(\bA-\bA')\bB_{i+1}\bB_{i+1}^T + 2\lambda_3(\bA-\bA')\|_F \label{eq:Lip2}\\
& \leq (\|\bX\bX^T\|_F\|\bB_{i+1}\bB_{i+1}^T\|_F + 2\lambda_3) \|\bA-\bA'\|_F, \label{eq:Lip3}
\end{align}
we can set
\begin{align}
\beta_i & = \|\bA_i^T\bX^T\bX\bA_i\|_F \\
\alpha_i & = \|\bX\bX^T\|_F\|\bB_{i+1}\bB_{i+1}^T\|_F + 2\lambda_3
\end{align}

Finally:
\begin{align}
\widetilde{\bB}_i & = \bB_i + \omega_i^B (\bB_i - \bB_{i-1}) \\
\widetilde{\bA}_i & = \bA_i + \omega_i^A (\bA_i - \bA_{i-1}) \\
\omega_i^A, \omega_i^B & \in [0,1)
\end{align}
are extrapolations which are used to accelerate the convergence. (See Section 4.3 of~\citep{Par14} for more details.) For our algorithm, we choose $$\omega_i^A = \min(\frac{t_{i-1} - 1}{t_i}, \delta_{\omega}\frac{\sqrt{\alpha_{i-1}}}{\sqrt{\alpha_{i}}}); \quad \omega_i^B = \min(\frac{t_{i-1} - 1}{t_i}, \delta_{\omega}\frac{\sqrt{\beta_{i-1}}}{\sqrt{\beta_{i}}}); \quad t_i = (1+\sqrt{1+4t_{i-1}^2})/2, \quad t_0=1,$$ for some $\delta_{\omega}<1$. The inequalities $\omega_i^A < \frac{\sqrt{\alpha_{i-1}}}{\sqrt{\alpha_{i}}}$ and $\omega_i^B < \frac{\sqrt{\beta_{i-1}}}{\sqrt{\beta_{i}}}$ are necessary for establishing  convergence.

The prox-linear updates need an extra step. If $f(\bA_{i+1},\bB_{i+1})\geq f(\bA_{i},\bB_{i})$, i.e., no descent, the two optimization steps are repeated with no extrapolation ($\omega_i=0$). This is also necessary for convergence.
\end{itemize}

\subsection{Theoretical Results}
The theoretical aspects of the alternating minimization (also known as block coordinate descent) algorithms have been extensively studied in a wide range of settings with different convexity and differentiability assumptions~\citep{Pow73, Gri00, Tse01, Tse09, Att09, Att10, Xu13, Xu14}. A full review of this vast literature is beyond the scope of this paper. Here, we focus on the recent work of Xu and Yin~\citep{Xu13} which considers a setting that matches the problem we study in this paper (block multi-convex objective with non-smooth penalty). Xu and Yin show that the solutions of proximal and prox-linear updates described above converge to a stationary point of the objective function and if the starting point is close to a global minimum, the alternating scheme converges globally. This conclusion is possible because the objective function in our problem is the sum of a real analytic function and a semi-algebraic function and satisfies the  Kurdyka--Lojasiewicz (KL) property. However, their results on convergence do not apply to the basic updates, because although the objective $\|\bY-\bX\bA\bB\|_F^2$ is convex in $\bB$, it is not strongly convex. In this section, we first present the results for proximal and prox-linear updates from~\citep{Xu13} and then provide some novel results for the basic update. 

\subsubsection{Definitions and convergence of the objective}

\begin{mydef}
$\bC^* = (\bA^*, \bB^*)$ is called a {\it partial optimum} of $f$ if 
\begin{eqnarray}
f(\bA^*,\bB^*) \ \ \leq & \!f(\bA^*,\bB ), \ \forall \bB\in\mathbb{R}^{m\times q} \\ \text{and} \quad f(\bA^*,\bB^*) \ \   \leq & \! f(\bA,\bB^* ), \ \ \forall \bA\in\mathbb{R}^{n\times m}.
\end{eqnarray}
\end{mydef}
Note that a stationary point of the objective is a partial optimum but the reverse does not generally hold; see Section 4 of~\citep{Wen12} for an example.
\begin{mydef}
A point $\bC^*$ is an {\it accumulation point} or a {\it limit point} of a sequence $\{\bC_i\}_{i\in\mathbb{N}}$, if for any neighbourhood $V$ of $\bC^*$, there are infinitely many $j\in\mathbb{N}$ such that $\bC_j\in V$. Equivalently, $\bC^*$ is the limit of a subsequence of $\{\bC_i\}_{i\in\mathbb{N}}$.
\end{mydef}


\begin{myprop}\label{proposition}
The sequence $f(\bA_i, \bB_i)$ generated by Algorithm 1 converges monotonically.
\end{myprop}
The value of  $f$ is always positive and is reduced in each of the two main steps of Algorithm 1. Thus, it is guaranteed that the stopping criterion of Algorithm 1 will be reached. 

\subsubsection{Proximal and prox-linear updates}
Since the sequence of solutions generated by the alternating minimization stays in a bounded, closed, and hence compact set, it has at least one accumulation point. Assuming that the parameters for the updates are chosen as explained above, we can state the following theorem about the accumulation points of the sequence of solutions. 
\begin{mytheo}
(Theorem 2.3 from~\citep{Xu13}). For a given starting point $\bC_0 = (\bA_0, \bB_0)$, let $\{\bC_i\}_{i\in\mathbb{N}}$ denote the sequence of solutions generated by proximal or prox-linear updates. Then, any accumulation points of the sequence $\bC_i$ is a partial minimum of $f$.
\end{mytheo}

The next theorem states that under some assumptions, if the sequence $\bC_i$ has a finite accumulation point, it converges to that point. 

\begin{mytheo}
(Theorem 2.8 from~\citep{Xu13}). For a given starting point $\bC_0 = (\bA_0, \bB_0)$, let $\{\bC_i\}_{i\in\mathbb{N}}$ denote the sequence of solutions generated by proximal or prox-linear updates and assume that it has a finite accumulation point $\bC^*$. Assuming that $\nabla(\frac{1}{2} \| \bY\!-\!\bX\bA\bB\|_F^2 \!+\! \lambda_3 \|\bA\|_{F}^2)$ is Lipschitz continuous and $f$ satisfies the Kurdyka--Lojasiewicz (KL) property at $\bC^*$, $\{\bC_i\}_{i\in\mathbb{N}}$ converges to $\bC^*$. 
\end{mytheo}
Both of these conditions are held. From (\ref{eq:Lip1})--(\ref{eq:Lip3}) we get the Lipschitz continuity. Also since $f$ is the sum of real analytic and semi-algebraic functions, it satisfies the KL property (see Section 2.2 of~\citep{Xu13} for details of the KL property). 

Finally, it is shown in ~\citep[Lemma 2.6]{Xu13} that if the staring point is sufficiently close to a global minimizer of $f$, then the sequence of solutions will converge to that point.

\subsubsection{Basic updates}
The results of the previous section do not hold for the basic updates because  the objective is not strongly convex in $\bB$. In this part, we provide some novel results for the basic update with the proofs in the Appendix. 
\begin{mytheo}\label{theo:uniqueness}
If the entries of $\bX \in \mathbb{R}^{n \times p}$ are drawn from a continuous probability distribution on $\mathbb{R}^{np}$, then: \\ (i) The solution of  (\ref{eq:opt1}) is unique if $\bA_i$ is full rank. \\ (ii) The objective of (\ref{eq:opt2}) is strongly convex and its solution, if one exists, is unique.
\end{mytheo}

In classical LASSO, the condition on the entries of $\bX$ is sufficient to achieve solution uniqueness~\citep{Tib13}. For LASSO, the continuity is used to argue that the columns of $\bX$ are in {\em general position} with probability 1 (see ~\ref{sec:appThm1}, Definition~\ref{def:generalposition} for a formal definition of general position). The affine span of the columns of $\bX$, $\{\bX_1,\ldots,\bX_{k+1}\}$, has Lebesgue measure $0$ in $\mathbb{R}^n$ for a continuous distribution on $\mathbb{R}^n$, so there is zero probability of drawing $\bX_{k+2}$ in their span. If we multiply $\bX$ by a matrix $\bA$ with full column rank, we retain the same property, and thus the solution of (\ref{eq:opt1}) is unique if $\bA_i$ is full rank. Also, since the objective function is strictly convex in $\bA$ (due to the elastic net property), if (\ref{eq:opt2}) has a solution, its solution is unique. Next, we study the properties of $\widehat{\bC} = (\widehat{\bA}, \widehat{\bB})$ at convergence.

\begin{mytheo}\label{theo:partial}
Assume that the entries of $\bX \in \mathbb{R}^{n \times p}$ are drawn from a continuous probability distribution on $\mathbb{R}^{np}$. For a given starting point $\bA_0$, let $\{\bC_i\}_{i\in\mathbb{N}}$ denote the sequence of solutions generated by Algorithm 1. Then: \\
(i) $\{\bC_i\}_{i\in\mathbb{N}}$ has at least one accumulation point. \\
(ii)  All the accumulation points of $\{\bC_i\}_{i\in\mathbb{N}}$ are partial optima and have the same function value. \\
(iii) If $\bB$ is full rank for all accumulation points of $\{\bC_i\}_{i\in\mathbb{N}}$, then:
\begin{equation} \label{eq:converge}
\lim_{i \rightarrow \infty} \| \bC_{i+1} - \bC_i \| = 0,
\end{equation}
\end{mytheo}

Part (i) follows from the fact that the solutions produced by Algorithm 1 are contained in a bounded, closed (and hence compact) set. Although Algorithm 1 converges to a specific value of $f$, this value can be achieved by different values of $\bC$. Thus, the sequence $\bC_i$ can have many accumulation points. Part (ii) of Theorem~\ref{theo:partial} shows that any accumulation point is a partial optimum. Proposition~\ref{proposition} implies that for any given starting point, all the associated accumulation points have the same $f$ value. Under the assumption that $\bB$ is full rank for all accumulation points of $\{\bC_i\}_{i\in\mathbb{N}}$, part (iii) provides a guarantee that the difference between successive solutions of the algorithm converges to zero, for both the factor and loading matrices. Although the condition in (\ref{eq:converge}) does not guarantee the convergence of the sequence $\{\bC_i\}_{i\in\mathbb{N}}$, it is close enough for practical purposes. Also, note that for finding the number of factors, we require  both $\bA$ and $\bB$ to be full rank for the final solution. When $\bB$ is full rank, the solutions to both (\ref{eq:opt1}) and (\ref{eq:opt2}) are unique and thus $\bA$ and $\bB$ will not change in the following iterations, i.e., convergence.

\subsection{Complete Algorithm}
Now, we propose Algorithm 2 to solve the  optimization problem described in (\ref{eq:sub2}--\ref{eq:sub1}) to find the number of latent factors as well as the factor and loading matrices. Optimal values of $\lambda_1$, $\lambda_2$, $\lambda_3$ are found via  5-fold cross-validation. Although the performance of our algorithm with different updates is roughly similar, our simulations show that the prox-linear updates provide the best results both in terms of prediction accuracy and speed of convergence. This could be due to the local extrapolation which helps the algorithm avoid small neighbourhoods around certain local minima. In the following, we present results for the prox-linear updates. 

\begin{algorithm}[!ht]
\caption{Sparse Multivariate Factor Regression (SMFR) via Alternating Minimization }
\begin{spacing}{1.1}
\begin{algorithmic}[1]
\State \textbf{Input:} Training Set $\bX_{n \times p}, \bY_{n \times q}, \lambda_1, \lambda_2$
\State \textbf{Output:} Solution of problem ~(\ref{eq:sub2}--\ref{eq:sub1}): $  \widehat{\bA}, \widehat{\bB}, \widehat{m}$
\State $m \gets r$ \Comment{$r$: {\it upper bound on the number of factors}}
\While {true}
\State $\bA \gets \bA_0 \in \mathbb{R}^{p \times m}$, $i \gets 0$ 
\While {value of $f(\bA, \bB)$ not converged}
\State \hspace*{-0.04 in}$\bB_{i+1}\!\gets\! \text{update $\bB$ with $\bA$ fixed at $\bA_i$}$
\State \hspace*{-0.07 in} $\bA_{i+1}\!\!\gets\!\! \text{update $\bA$ with $\bB$ fixed at $\bB_{i+1}$}$
\State $i \gets i+1$ 
\EndWhile
\State $\widehat{\bA}, \widehat{\bB} \gets $ \ values of $\bA$ and $\bB$ at convergence
\If {$ \text{rank}(\widehat{\bA}) < m \textbf{ or } \text{rank}(\widehat{\bB}) < m$}
\State $m \gets m-1$
\Else
\State \textbf{break}
\EndIf
\EndWhile
\end{algorithmic}
\end{spacing}
\end{algorithm}

\section{Fully Sparse PCA}\label{SPCA}
In~\citep{Zou06s}, Zou, Hastie, and Tibshirani propose a sparse PCA, arguing that in regular PCA
``each principal component is a linear combination of all the original variables, thus it is difficult to interpret the results''. Assume that we have a data matrix $\bX_{n\times p}$ with the following SV decomposition: $\bX = \bU\bD\bV^T$. The principal components are defined as $\bZ = \bU\bD$ with the corresponding columns of $\bV$ as the loadings. In~\citep{zou06}, Zou et al. show that solving the following optimization problem leads to exact PCA:
\begin{equation*}
(\widehat{\bA}, \widehat{\bB})\!=\!\argmin_{\bA, \bB}\!\|\bX\!-\!\bX\bA\bB\|_F^2+\lambda\!\sum_i \|\bA_i\|_2^2  \textrm{ s.t. }  \bB\bB^T\!\!=\!{\bf I},
\end{equation*}
where $\bA_i$ denotes the $i$'th column of $\bA$. Thus, we have $\widehat{\bA}_i \propto \bV_i$, and  $\bX\widehat{\bA}_i$ corresponds to the $i$'th principal component. Sparse PCA (SPCA) is introduced by adding an $\ell_1$ penalty on matrix $\bA$ to the objective function. Zou et al. propose an alternating minimization scheme to solve this problem~\citep{Zou06s}.

The ordinary principal components are uncorrelated and their loadings are orthogonal. SPCA imposes sparsity on the construction of the principal components. Here sparsity means that each component is a combination of only a few of the variables. By enforcing sparsity, the principal components become correlated and the loadings are no longer orthogonal. On the other hand, SPCA assumes, like regular PCA, that the contributions of these components are orthonormal ($\bB\bB^T\!=\!\bI$). In our algorithm, if we replace $\bY$ with $\bX$, i.e., regressing $\bX$ on itself, we get a similar algorithm. However, our work differs in two  ways. First, we also impose sparsity on the contributions of principal components. This comes at the expense of higher computational costs, but results in more interpretable results. Also, the contributions will not be orthonormal anymore. However, by the full rank constraint we impose on the two matrices, we make sure that the principal components and their contributions are linearly independent. Moreover, the algorithm provides a mechanism for learning from the data how many principal components are sufficient to explain the data.

\section{Simulation Study} \label{sec:synth}
In this section, we use synthetic data to compare the performance of our algorithm with several related multivariate regression methods reviewed in the introduction.

\subsection{Simulation Setup}
We generate the synthetic data in accordance with the model described in~(\ref{eq:mv2}), $\bY = \bX \bD + \bE$, where $\bD=\bA\bB$. First, we generate an $n\times p$ predictor matrix, $\bX$, with  rows independently drawn from $\mathcal{N}({\bf 0}, \bsig_X)$, where the $(i,j)$-th element of $\bsig_X$ is defined as $\sigma^X_{i,j} = 0.7^{|j-i|}$. This is a common model for predictors in the  literature~\citep{Yua07, pen10,Roth10}. The rows of the $n\times q$ error matrix are sampled from $\mathcal{N}({\bf 0}, \bsig_N)$, where the $(i,j)$-th element of $\bsig_N$ is defined as $\sigma^X_{i,j} = \sigma_n^2 \cdot 0.4^{|j-i|}$. The value of $\sigma_n^2$ is varied to attain different levels of signal to noise ratio (SNR). Each row of the $p\times m$ matrix $\bA$ is chosen  by first randomly selecting $m_0$ of its elements and sampling them from $\mathcal{N}(0,1)$ and then setting the rest of its elements to zero. Finally, we generate the $m \times q$ matrix $\bB$ by the element-wise product of $\bB = \bU \circ \bW$, where the elements of $\bU$ are drawn independently from $\mathcal{N}(0,1)$ and elements of $\bW$ are drawn from Bernoulli distribution with success probability  $s$.

We evaluate the performance of a given algorithm with three different metrics. We evaluate the predictive performance over a test set $(\bX_{test}, \bY_{test})$, separate from the training set, in terms of the mean-squared error:
\begin{equation}\label{eq:MSE}
\textrm{MSE} = \frac{\| \bX_{test}\widehat{\bD} - \bY_{test} \|_F^2}{nq},
\end{equation}
where $\widehat{\bD}$ is the estimated coefficient matrix. In our case, we have $\widehat{\bD} = \widehat{\bA}\widehat{\bB}$. We also compare different algorithms based on their signed sensitivity and specificity of the support recognition:
\vspace*{-0.03 in}
\begin{align*}
\textrm{Signed  Sensitivity}  & = \frac{\sum_{i,j}{\bf 1[} d_{i,j} \cdot \widehat{d}_{i,j} > 0{\bf ]}}{\sum_{i,j} {\bf 1}[d_{i,j}\neq 0]}, \\
\textrm{Specificity}  & = \frac{\sum_{i,j}{\bf 1[} d_{i,j} = 0{\bf ]}\cdot {\bf 1[} \widehat{d}_{i,j} = 0{\bf ]}}{\sum_{i,j} {\bf 1}[d_{i,j}= 0]},
\end{align*}
where ${\bf 1}$ represents the indicator function.


We compare the performance of our algorithm, SMFR, with many other algorithms reviewed in Section~\ref{sec:rel} as well as a baseline algorithm with a simple ridge penalty.  We consider three different regimes: (i) high-dimensional problems with few instances (50) compared to the number of predictors or responses (50, 100, or 150); (ii) problems with increased number of instances ($p,q<n$); and (iii) problems where the structural assumption of our technique is violated. In the first regime, which is of most interest to us due to high-dimensionality, we explore different parameter settings. The values of $\sigma_n$ and $s$ affect the SNR---lower values of $\sigma_n$ and higher values of $s$ correspond to higher values of SNR (e.g., $\sigma_n=5, s=0.1$ corresponds to a very low SNR).
In regime (iii), we violate the assumption about the structure of the coefficient matrix, i.e., $\bD=\bA\bB$, in two ways. In the first case, $\bD$ has an element-wise sparsity with a density of $20\%$; in the second, it has row-wise sparsity where $60\%$ of the rows are all zeros and the rest have $30\%$ non-zero elements. We consider these cases to compare our algorithm with others in an 
unfavourable setting.

\begin{table}[ht!]
\centering
\tabcolsep=0.1cm
\hspace*{-0.3 in}
\begin{tabular}{c c c c c c  c || c c c c c c c c}
\multicolumn{7}{c}{\underline{Parameters}} & \multicolumn{8}{c}{\underline{MSE over test set}} \\ 
 
$n$ & $p$ & $q$ & $m$ & $m_0$ & $\sigma_n$ & $s$ & SMFR & LASSO & $\ell_1/\ell_2$~\citep{Obo11} & SRRR~\citep{Chen12} & RemMap~\citep{pen10} & SPLS~\citep{Chun10} & Trace~\citep{Ji09} & Ridge\\ 
\hline \hline
50 & 150 & 50 & 10 & 1 & 3 & 0.2 & \spcell{0.070 \\ (0.004)} &  \spcell{0.083 \\ (0.005)} & \spcell{0.090 \\ (0.005)} & \spcell{0.084 \\ (0.005)} & \spcell{0.083 \\ (0.005)} & \spcell{0.091 \\ (0.007)} & \spcell{0.088 \\ (0.004)}& \spcell{0.089 \\ (0.004)} \\ 
\hdashline
& & & 10 &1 & 3   & 0.4 &  \spcell{0.078 \\ (0.007)} &  \spcell{0.104 \\ (0.008)} & \spcell{0.105 \\ (0.007)} & \spcell{0.099 \\ (0.007)} & \spcell{0.104 \\ (0.008)} & \spcell{0.110 \\ (0.008)} & \spcell{0.110 \\ (0.007)}& \spcell{0.111 \\ (0.006)}\\ 
\hdashline
& & &10 & 1& 5   & 0.2 &  \spcell{0.110 \\ (0.004)} &  \spcell{0.118 \\ (0.005)} &  \spcell{0.133 \\ (0.004)} & \spcell{0.117 \\ (0.005)} & \spcell{0.123 \\ (0.004)} & \spcell{0.122 \\ (0.007)} & \spcell{0.115 \\ (0.005)}& \spcell{0.122 \\ (0.006)}\\ 
\hdashline
& & &15 & 2& 3   & 0.2  & \spcell{0.071 \\ (0.003)}& \spcell{0.108 \\ (0.006)}&\spcell{0.112 \\ (0.007)}&\spcell{0.109 \\ (0.008)}&\spcell{0.107 \\ (0.006)}& \spcell{0.114 \\ (0.008)}& \spcell{0.109 \\ (0.006)}& \spcell{0.110 \\ (0.008)}\\
\hdashline
50 & 100 & 100 & 10 & 1 & 5 & 0.1 &\spcell{0.068 \\ (0.001)}&\spcell{0.070 \\ (0.002)}&\spcell{0.092 \\ (0.002)}&\spcell{0.071 \\ (0.002)}&\spcell{0.075 \\ (0.002)}& \spcell{0.073 \\ (0.002)}& \spcell{0.071 \\ (0.002)}& \spcell{0.074 \\ (0.002)}\\
\hline
500 & 150 & 50 & 10 & 1 & 3 & 0.2 &\spcell{0.0172 \\ (0.0001)}&\spcell{0.0180 \\ (0.0001)}&\spcell{0.0198 \\ (0.0001)}&\spcell{0.0176 \\ (0.0001)}&\spcell{0.0184 \\ (0.0001)}&\spcell{0.0216 \\ (0.0007)} & \spcell{0.0183 \\ (0.0002)}& \spcell{0.0187 \\ (0.0001)}\\
\hdashline
500 & 100 & 100 & 10 & 1 & 5 & 0.3 &\spcell{0.0202 \\ (0.0001)}&\spcell{0.0209 \\ (0.0002)}&\spcell{0.0222 \\ (0.0001)}&\spcell{0.0204 \\ (0.0001)}&\spcell{0.0214 \\ (0.0001)}& \spcell{0.0222 \\ (0.0003)} & \spcell{0.0208 \\ (0.0001)}& \spcell{0.0213 \\ (0.0002)}\\
\hline
50 & 100 & 100 & \multicolumn{2}{c}{\spcell{\it element-wise \\ \it sparsity}} & 5 & ---   & \spcell{0.079 \\ (0.001)} & \spcell{0.078 \\(0.001)}&\spcell{0.096 \\(0.002)}&\spcell{0.080 \\ (0.001)}&\spcell{0.085 \\ (0.001)}& \spcell{0.081 \\ (0.001)}& \spcell{0.078 \\ (0.002)}& \spcell{0.081 \\ (0.001)}\\
\hdashline
50 & 150 & 50 & \multicolumn{2}{c}{\spcell{\it row-wise \\ \it sparsity}} & 3 & ---   & \spcell{0.082 \\ (0.004)} & \spcell{0.076 \\(0.003)}& \spcell{0.080 \\(0.003)}&\spcell{0.079 \\(0.003)}&\spcell{0.075 \\(0.003)}& \spcell{0.083 \\ (0.003)}& \spcell{0.081 \\ (0.001)}& \spcell{0.102 \\ (0.004)}\\
\hline
\end{tabular} 
\vspace*{0.1 in}
\caption{Comparison of six algorithms for different setups. We report mean and standard deviations of the MSE over the test sets, based on 20 simulation runs.\label{tab:comp}}
\end{table}



\subsection{Results}
\paragraph{Predictive performance} The means and standard deviations of MSE for different algorithms are presented in Table~\ref{tab:comp}. We use five-fold cross-validation to find the tuning parameters of all algorithms.  We set $r$, the maximum  number of factors, to 20. For the first two simulation regimes, our algorithm outperforms the other algorithms and results in lower MSE means and standard deviation. The improvements are more significant in the high-dimensional settings with high SNR. However, in settings with low SNR ($\sigma_n=5, s=0.1$) or high number of instances ($n=500$), we still observe lower errors for SMFR. On average, our algorithm reduces the test error by $13.2\%$ compared with LASSO, $21.4\%$ compared with $\ell_1/\ell_2$, $12.3\%$ compared with SRRR, $15.2\%$ compared with RemMap,  $19.4\%$ compared with SPLS, $39.1\%$ compared with Trace, and $16.7\%$ compared with Ridge. 

In the last two simulations, where the assumed factor structure is abandoned, our algorithm has no advantage over simpler methods with no factor structure (such as LASSO) and gives a higher error.

\paragraph{Variable selection} In Figure~\ref{fig:sesp}, we compare the average signed sensitivity and specificity of different algorithms (based on 20 simulation runs) as the number of instances increase (other parameters kept fixed). We observe that our algorithm has higher sensitivity and specificity. This effect for specificity is reduced as the number of instances increases.  This shows that our algorithm is more advantageous in  high-dimensional settings where the number of instances is comparable to or less than the number of predictors/responses. Although we only show the plots for a specific parameter setting, the results are similar for other parameters. 

\begin{figure}[!ht]
\begin{center}
\subfigure{%
\includegraphics[scale=0.17]{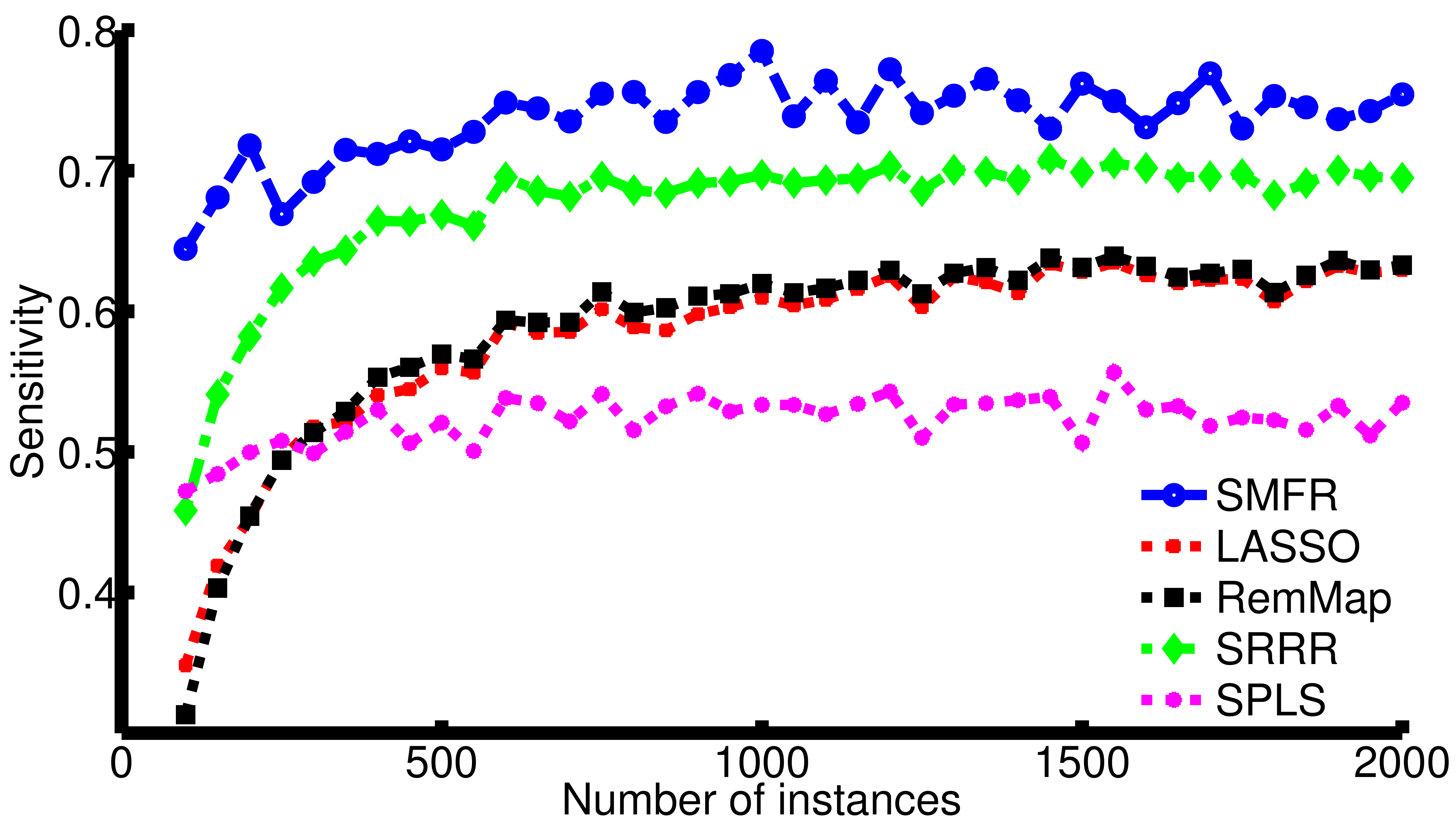}
}
\subfigure{%
\includegraphics[scale=0.17]{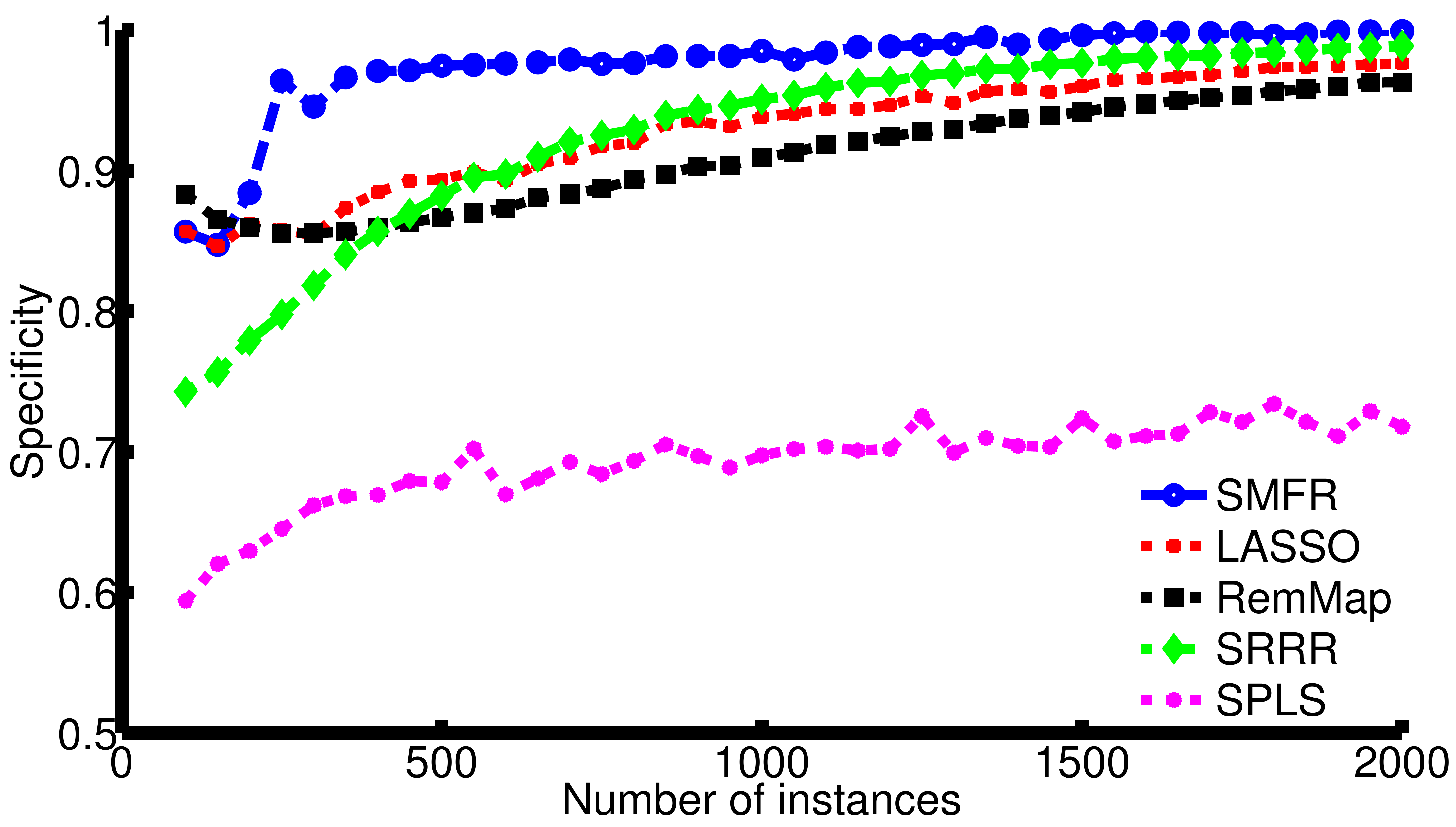}
}\caption{Sensitivity and specificity comparison of different algorithms as the number of instances increases. ($p=150, q=50,m=10,m_0=1,\sigma_n=3,s=0.3$)\label{fig:sesp}}
\end{center}
\end{figure}

\paragraph{Number of latent factors} In Table~\ref{tab:nfac}, we compare the number of estimated factors for the three algorithms that perform dimensionality reduction. For SRRR and SPLS, we find the number of factors by 5-fold cross-validation. 
The true number of factors is mentioned in the fourth column. We observe that our algorithm provides better estimates of the number of factors. 

\begin{table}[!ht]
\begin{center}
\bgroup
\begin{tabular}{@{\hspace{-.1em}}c @{\hspace{.3em}}c @{\hspace{.3em}}c @{\hspace{.3em}}c @{\hspace{.3em}}c @{\hspace{.3em}}c @{\hspace{.3em}}c @{\hspace{.3em}}| @{\hspace{.1em}}c| @{\hspace{.1em}}c|@{\hspace{.1em}} c}
$n$ &$p$ &$q$& $m$&$m_0$&$\sigma_n$&$s$ & ~SMFR & \ SRRR & \ SPLS \\
\hline
50 & 150&50&10& 1& 3& 0.2 & \ 10, 10.3, 1.1 & \ 8, 8.5, 1.4 & \ 14, 13.7, 3.1 \\
50 & 150&50&10& 1& 3& 0.4 & \ 10, 10.4, 1.3 & \ 10, 9.4, 1.2 & \ 18, 16.6, 3.1  \\
50& 150&50& 10& 1& 5& 0.2 & \ 11, 11.1, 1.8 & \ 7, 6.7, 2.4 & \ 7, 8.2, 2.9\\
50 & 150&50&15& 2& 3& 0.2 & \ 14, 13.5, 1.1 & \ 13, 12.6, 2.3 & \ 19, 19.1, 1.1 \\
50& 100& 100& 10& 1& 5& 0.1 & \ 7, 7.3, 2.8 & \ 6, 6.1, 3.0 & \ 4, 4.4, 1.3\\
500& 150& 50& 10& 1& 3& 0.2 & \ 10, 10, 0.0 & \ 10, 9.9, 0.2 & \ 15, 15, 0.0\\
500& 100& 100& 10& 1& 5& 0.3 & \ 10, 10.1, 0.9 & \ 10, 9.9, 0.2 & \ 15, 15, 0.0\\
\hline
\end{tabular}
\caption{Median, mean, and standard deviation of estimated number of factors (based on 20~runs).\label{tab:nfac}}
\egroup
\end{center}
\end{table}

\paragraph{Computation time} We also compare the computation time of different algorithms for the parameter settings corresponding to the first row of Table~\ref{tab:comp}. We report the median computation time (on a PC with 16GB RAM and quad-core CPU at 3.4GHz) of each algorithm (excluding the cross-validation part) over 20 runs; SMFR (with prox-linear updates):~1.6s, SRRR:~5.6s, RemMap:~0.3s, SPLS:~0.3s, LASSO:~0.02s, and $\ell_1/\ell_2$:~0.33s. Since the algorithms are implemented using different programming languages and stopping criteria vary slightly, care must be taken when interpreting these results. The main message is that SMFR and SRRR have  larger computation times, since they
solve more complicated problems which involve estimating the number of factors, and the loading and factoring matrices. This  extra  information has value in itself and provides more insight about the structure of data. Also, SMFR is almost three times faster than SRRR.
 
 \paragraph{SMFR initialization} In all the experiments above, and the ones in the next section, we use a random initialization for our algorithm  where the elements of $\bA_0$ are sampled independently from $\mathcal{N}(0,1)$. We compare this initialization with two other methods which are based on the matrix factorization of other solutions: (i) using the SVD of LASSO solution, $\bD_{\text{LASSO}} = \bU\bS\bV^T$, setting $\bA_0 = \bU_{1:m} \bS_{1:m}$ and $\bB_0 = (\bV^T)_{1:m}$, where the subscript $1\!:\!m$ indicates choosing the first $m$ columns; and (ii) using the SVD of trace-norm solution, $\bD_{\text{Trace}} = \bU\bS\bV^T$, setting $\bA_0 = \bU_{1:m} \bS_{1:m}$ and $\bB_0 = (\bV^T)_{1:m}$. We compare these three initializations for the case where the data is generated according to the first regime of Table 1 and vary the level of noise and sparsity.  The results are shown in Table~\ref{tab:init}. The procedure to find the number of effective factors is the same for all three initializations.  Also, as a baseline comparison, we show the results for the usual LASSO regression in the last row. The results reported in Table~\ref{tab:init} are based on 20 runs (i.e., 20 different realizations of the data).  It would also be interesting to examine how the results change based on different random initializations for a fixed realization of the data.  In Table~\ref{tab:randinit}, we show the results for 5 realizations of the data. For each realization, we run our algorithm with 20 different random initializations and report the mean and standard deviation of MSE over the test set. For comparison, we also include the MSE achieved by the decomposition of LASSO and trace norm solutions.  
 
 The results show that our algorithm is very robust to the choice of the starting point. The results in both tables show that random initialization gives similar (or slightly better) results compared to more sophisticated initializations. Moreover, the very small standard deviations in the first row of Table~\ref{tab:randinit} and the similarity of the error for all three types of initializations show that for a given realization of the data, different initializations lead to very similar estimates by our algorithm -- another indication of the robustness of our algorithm to the choice of the starting point.

 \begin{table}[!h]
\begin{center}
\begin{tabular}{| c | c | c | c | c |}
\hline
  & $\sigma_n = 3, s = 0.2$ & $\sigma_n = 3, s = 0.3$ &  $\sigma_n = 5, s = 0.2$ & $\sigma_n = 5, s = 0.3$ \\ \hline
 Random Initialization& 0.070 (0.003) & 0.075 (0.003) & 0.110 (0.005) & 0.116 (0.005) \\  \hline
 LASSO Initialization&  0.071 (0.003) & 0.076 (0.004) & 0.112 (0.004) & 0.117 (0.004) \\  \hline 
 Trace Norm Initialization&  0.071 (0.004)& 0.076 (0.004) & 0.113 (0.004) & 0.119 (0.006)\\ \hline 
 LASSO & 0.085 (0.004) & 0.097 (0.005) & 0.119 (0.005) & 0.135 (0.007)\\ \hline
\end{tabular}
\vspace*{0.1 in}
\caption{Comparing different initialization techniques. ($p=150, q=50, n=50,$ and $m=10$)\label{tab:init}}
\end{center}
\end{table}

 \begin{table}[!h]
\centering
\label{my-label}
\begin{tabular}{|c||c|c|c|c|c|c|}
\hline
         Initialization &    Instance \#1 & Instance \#2 & Instance \#3 & Instance \#4 & Instance \#5\\ \hline \hline
{Random} &  0.071 (0.0004) & 0.069 (0.001) & 0.065 (0.0004) & 0.070 (0.0005) & 0.071 (0.0005)\\ \hline 
{LASSO} & 0.071 & 0.070 & 0.065& 0.070 & 0.072\\ \hline 
{Trace Norm} &  0.071 & 0.070 & 0.065 & 0.071 & 0.071\\ \hline
\end{tabular}
\vspace*{0.1 in}
\caption{MSE variation over 20 different random initializations for the same problem (first row of Table 1). \label{tab:randinit}}
\end{table}

\section{Application to Real Data}\label{sec:real}
We apply the proposed algorithm to real-world datasets and show that it exhibits better or similar predictive performance compared to state-of-the-art algorithms. We show that the factoring identified by our algorithm provides valuable insight into the underlying structure of the datasets.

\subsection{Montreal's bicycle sharing system (BIXI)}
The first dataset we consider provides information about Montreal's bicycle sharing system called BIXI. The data contains the number of available bikes in each of the 400 installed stations for every minute. We use the data collected for the first four weeks of June 2012. From this dataset we form the set of predictors and responses as follows. We allocate two features to each station corresponding to the number of arrivals and departures of bikes to or from that station for every hour. The learning task is to predict the number of arrivals and departures for all the stations from the number of arrivals and departures in the last hour (i.e., a vector autoregressive model). The choice of this model is a compromise between accuracy and complexity. Mathematically, we want to estimate $\bD$ such that $\bY \simeq \bX\bD$, where  $\bX_{t,j}$ and $\bX_{t,400+j}$ respectively show the number of arrivals and departures in hour $t$ at station $j$ and $\bY_{t,j}$ and $\bY_{t,400+j}$  respectively show the number of arrivals and departures in hour $t+1$ at station $j$.

We perform the prediction task on each of the four weeks. For each week, we take the data for the first 5 days (120 data points; Friday to Tuesday) as the training set (with the fifth day data as the validation set), and the last two days as the test set (48 data points; Wednesday and Thursday). We compare the algorithms performing dimensionality reduction in terms of their predictive performance on the test sets and the number of chosen factors in Table~\ref{tab:bixi}. We also include LASSO and $\ell_1/\ell_2$ as baseline algorithms. To avoid showing very small numbers, we present the value of MSE, defined in~(\ref{eq:MSE}), times $nq$.  In terms of the prediction performance, we observe that our algorithm outperforms the others in all 4 weeks. 
We can also compare the algorithms in terms of the number of chosen factors. For our algorithm, we set $r=15$ (the upper bound on the number of factors), and for the others, we do the cross-validation for 1 to 15 factors (15 different values). SMFR always uses all the 15 factors while the other two algorithms choose fewer factors (using cross-validation).\\ 

\begin{table}[h!]
\centering
\tabcolsep=0.11cm
\begin{tabular}{ c | c | c c c c c}
 week & & SMFR & SRRR  & SPLS & LASSO & $\ell_1/\ell_2$ \\
 \hline \hline
 \multirow{2}{*}{1} & error & 557.3 & 570.0  & 1661 & 580.4 &591.0\\
 \cline{2-7}
 &  \small factors & 15& 3& 7&---&--- \\
 \hline 
 \multirow{2}{*}{2} & error& 570.1 & 602.2  & 1888 & 610.9 &623.7\\
  \cline{2-7}
 &  factors & 15& 3& 6&---&--- \\
 \hline
 \multirow{2}{*}{3} & error & 618.8 & 641.9  & 2159 & 643.4 &657.8\\
  \cline{2-7}
 &  factors & 15& 7& 5&---&--- \\
 \hline
 \multirow{2}{*}{4} & error & 549.7 & 594.6  & 1621 & 594.9 &588.0\\
  \cline{2-7} 
 &  factors & 15& 3& 6& ---&--- \\
 \hline
\end{tabular}
\vspace*{0.1 in}
\caption{Total squared error (MSE$\times nq$) and the number of factors for BIXI dataset \label{tab:bixi}}
\end{table}


To gain more insight into the quality of predictions, we randomly choose two features in week 4 and plot the predictions made by different algorithms over the test set in Figure~\ref{fig:bixi_err}; the results for other features are similar. The y-axis represents the cumulative number of bikes. We observe that SMFR provides a better fit to the data.

\begin{figure}[!ht]
\begin{center}
\subfigure{%
\includegraphics[scale=0.15]{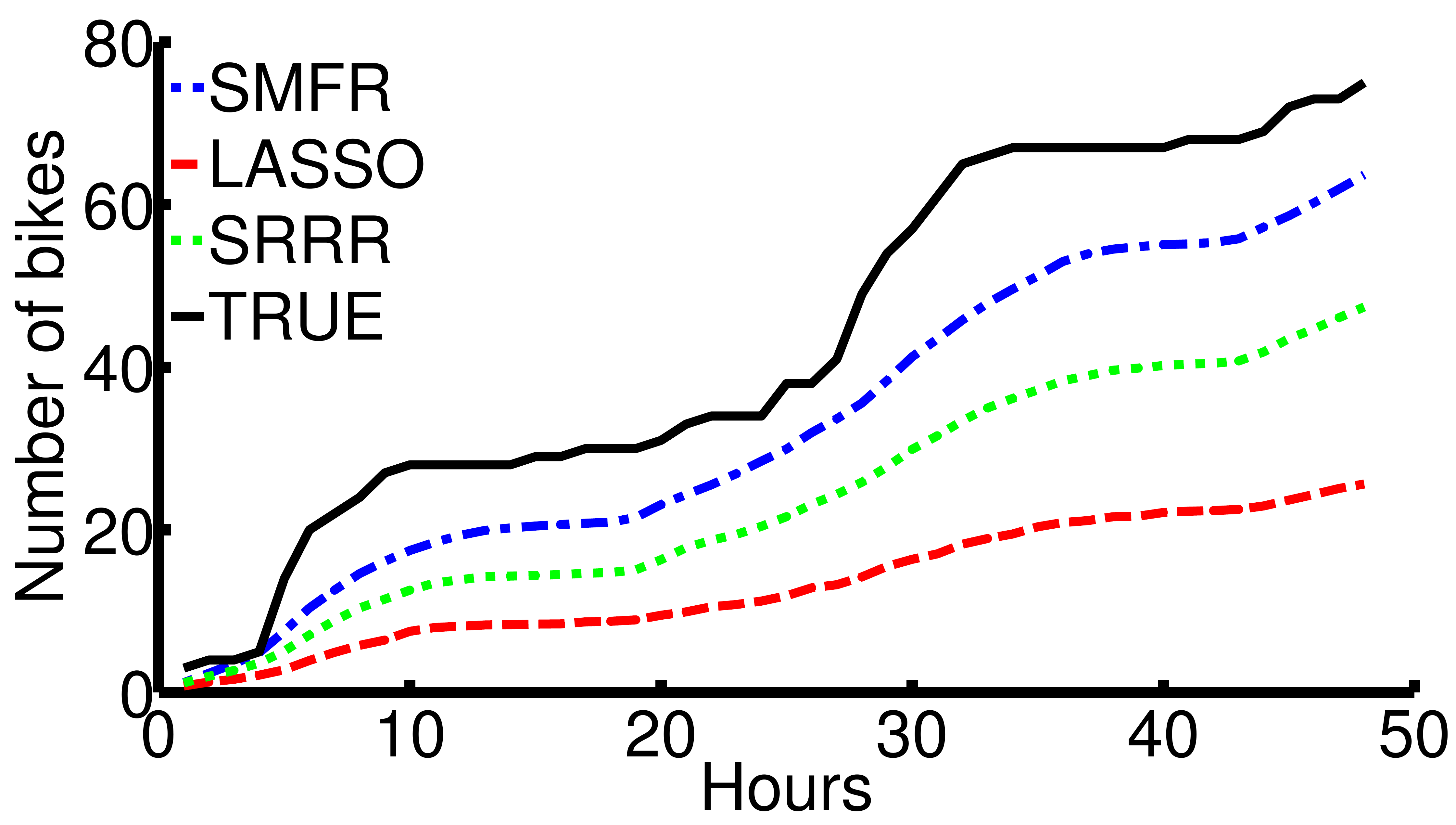}
}
\subfigure{%
\includegraphics[scale=0.15]{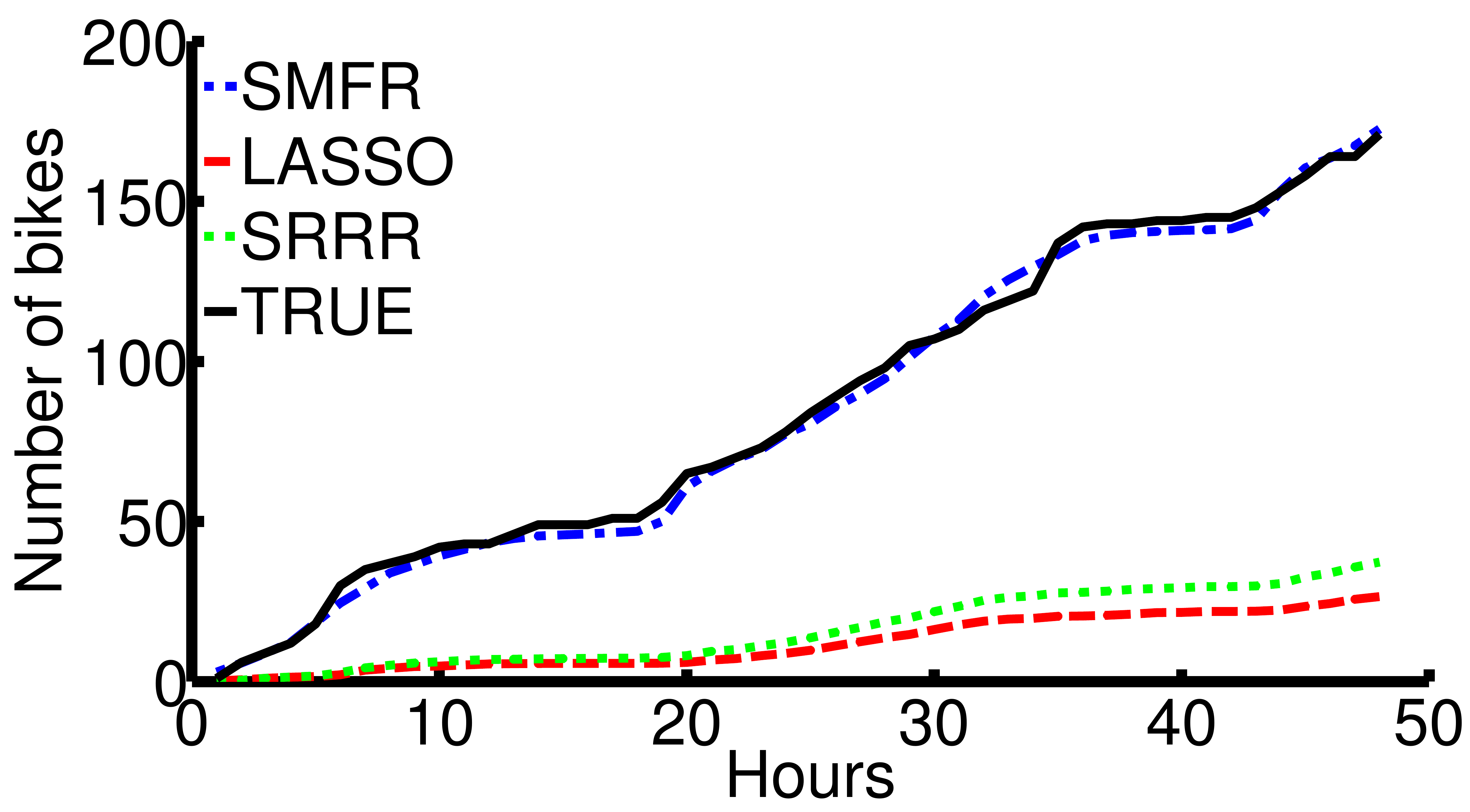}
}\caption{Comparing fit of different algorithms for two sample cumulative features. Proposed algorithm performs better than others for most stations. 
\label{fig:bixi_err}}
\end{center}
\end{figure}

\begin{figure}[!hb]
\begin{center}
\subfigure[Factor 1: populated residential areas to downtown]{%
\label{fig:bixifirst}
\includegraphics[scale=0.22]{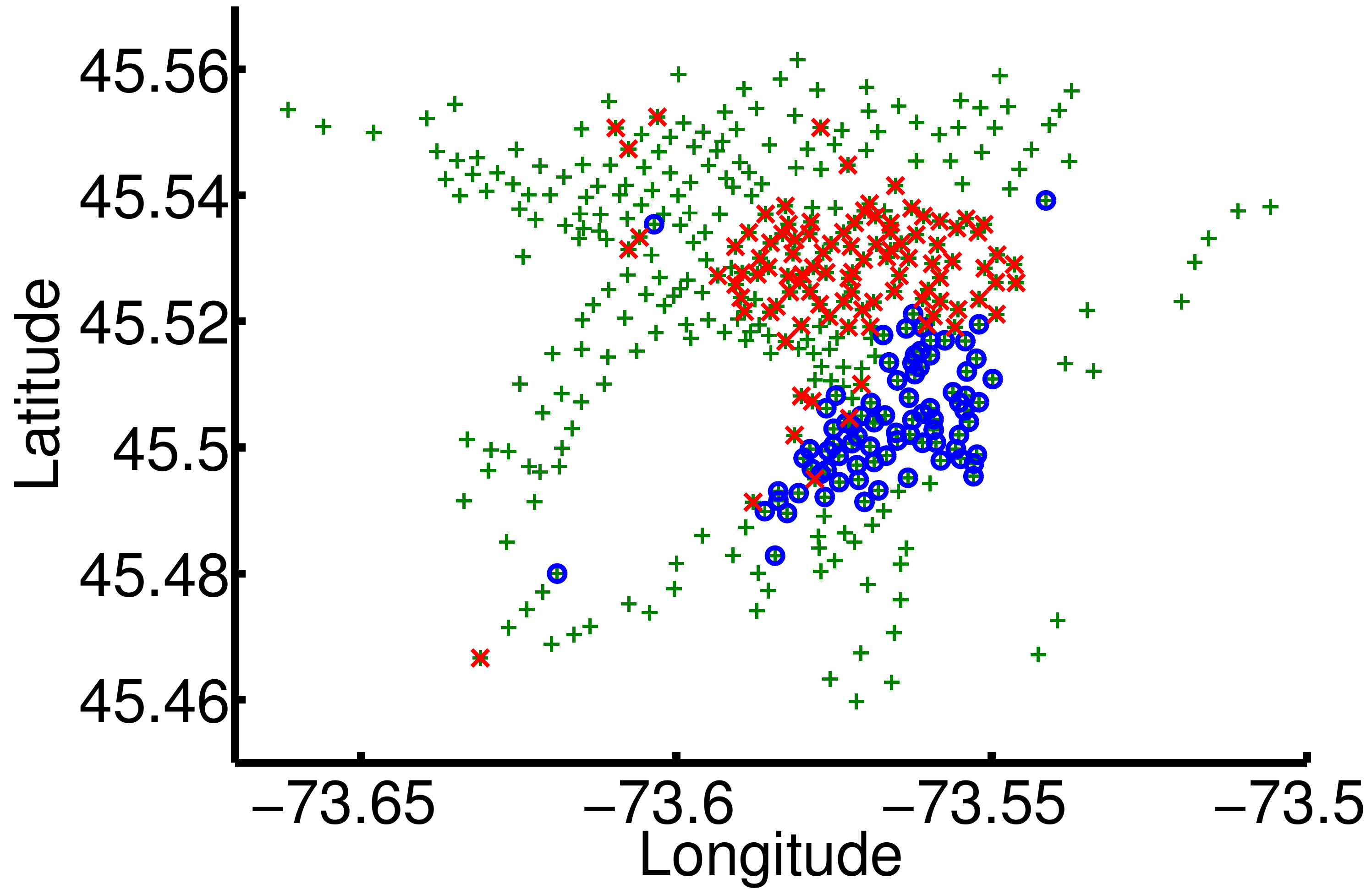}
}
\subfigure[Factor 2: peripheral parts of downtown to central, popular parts]{%
\label{fig:bixisecond}
\includegraphics[scale=0.22]{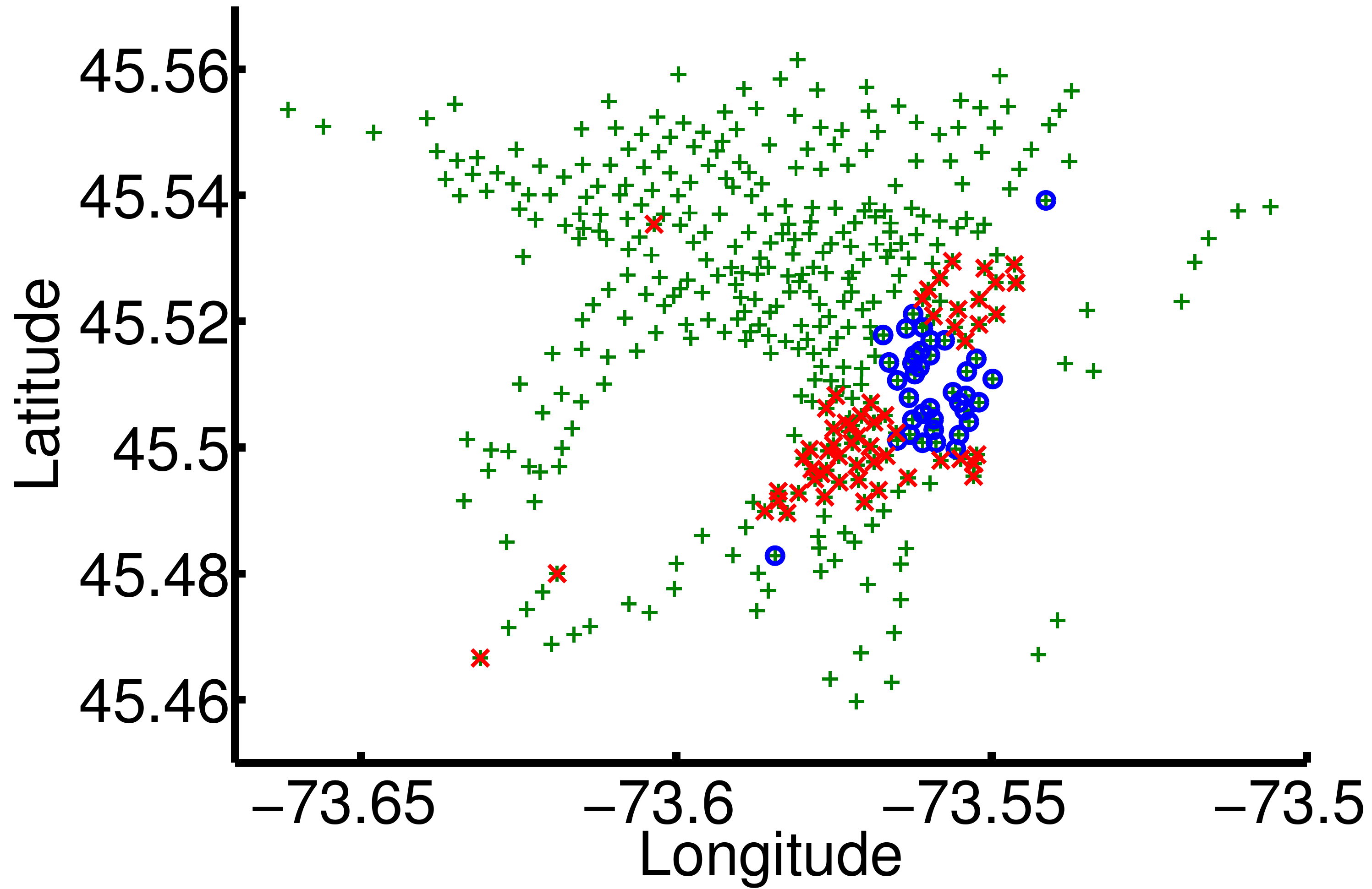}
}
\subfigure[Factor 3: flow inside the central downtown]{%
\label{fig:bixithird}
\includegraphics[scale=0.22]{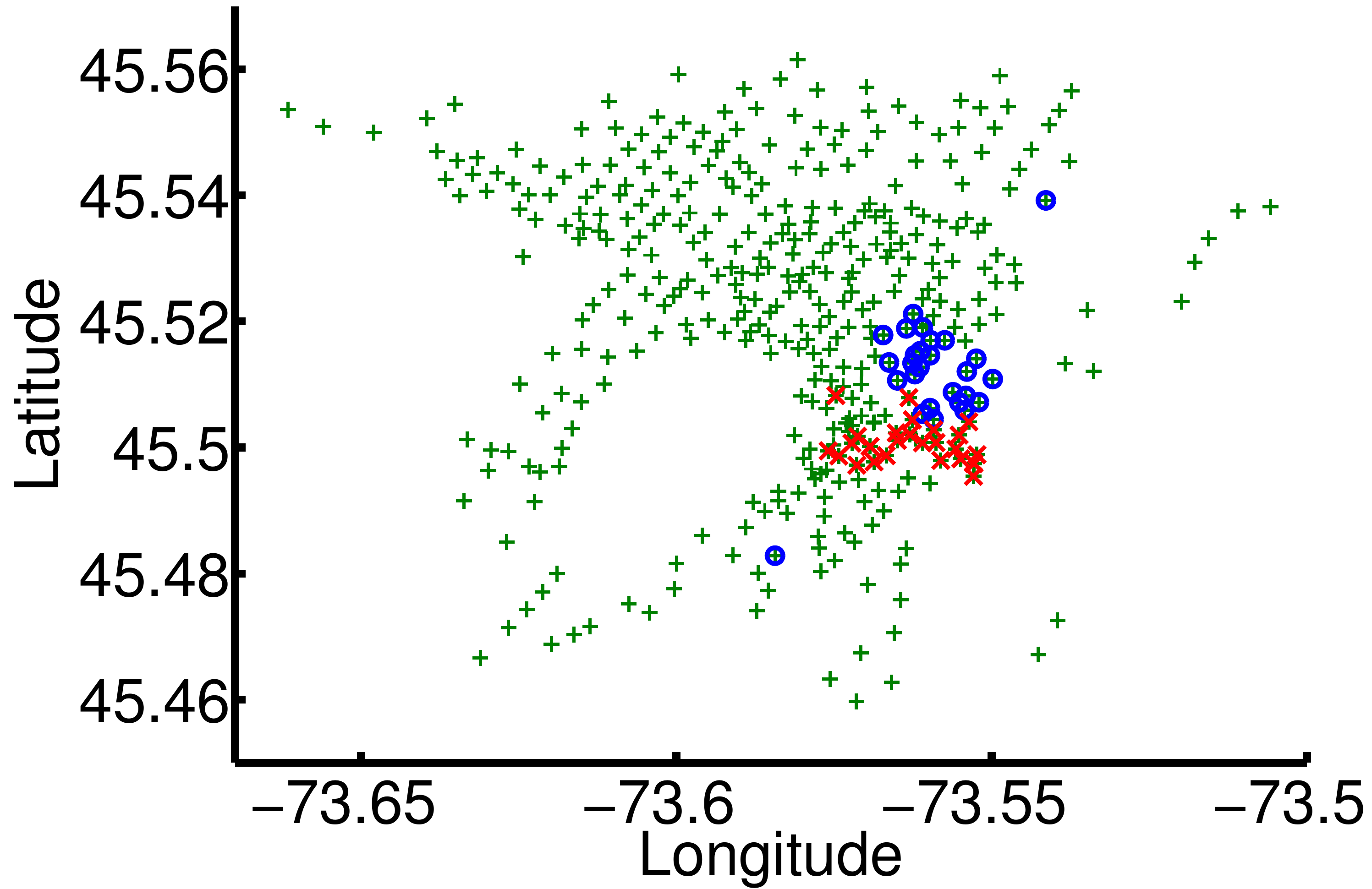}
}\caption{Three of the factors identified in the BIXI dataset by our algorithm. Green plus signs show the stations, red crosses show the departure features and blue circles show the arrival features of a station chosen in the factor. 
\label{fig:bixi}}
\end{center}
\end{figure}

We are using the data from the weekends as well as three weekdays to learn the prediction model, so it is reasonable to ask whether it is sensible to combine weekends and weekdays. Our preliminary data analysis indicated that the data on weekends are not particularly different from the other days, and thus we can include weekends in our training sets. For example, the correlation between the number of arrivals on weekends and the number of arrivals on weekdays across stations, is very high (around $0.9$) for all four weeks, showing that the relative level of activity in a station (compared to other stations) is similar for weekends and weekdays; i.e, if a station has relatively high number of arrivals on the weekdays, it is highly likely that it will have a relatively high number of arrivals on the weekends too. We have repeated the experiment after removing the weekends and obtained similar results.

To further investigate the variable selection of our algorithm, we run it on the whole data and examine the resulting factors.  Three of these are shown in Figure~\ref{fig:bixi}. In these figures, all the bike stations are shown with green plus signs. For each factor, we show its constituent features; red crosses and blue circles correspond respectively to the departure and arrival features of each station that are present in that factor. Examining these factors provide useful insight into the data. For instance, the factor in Figure~\ref{fig:bixifirst} shows that the departures from populated residential areas (The Plateau, Mile End, Outremont) and arrivals at downtown (Ville Marie) are combined together to form a factor. This agrees with the intuition that many people are taking  bikes to go from their homes to downtown where universities and businesses are located. The factor in Figure~\ref{fig:bixisecond} shows another strong effect which corresponds to the flow from the peripheries of downtown to more central locations (Place des Arts, Old Port). Many hotels and several universities are situated at the edge of downtown; numerous restaurants,  cafes and tourist sites are located in the centre, and several festivals  occurred there during June. The third factor represents the flow within this central part.

Our model is expected to provide a better fit to the BIXI data since its assumed structure matches the underlying structure of bike movements. It is expected  that generally, people ride from certain areas of the city to others. The factors capture this aggregated movement (similar to wavelets in the sense of smoothing over multiple individual sites). We expect the factors (matrix $\bA$) to be sparse because they capture movement from one region to another, and any stations outside these regions do not participate. We expect the loadings (matrix $\bB$) to be sparse because each station should only be predicted by those factors that involve it in terms of arrivals or departures. Therefore, as also confirmed by predictive performance, our proposed sparse, low-dimensional structure is a good fit to the data.

\subsection{S\&P 500 stocks}
We consider 294 companies from the S\&P 500~\citep{SP500} and collect their daily returns (percentage change in value from one day to the next) between March 1992 and December 2013 for a total of 5500 days. These returns are volatility adjusted using a GARCH model~\citep{fra11} and market adjusted by subtracting the market's average return for each day. We use the global industry classification standard~\citep{GICS} which categorizes all major public companies into 10 sectors.  Since there are very few companies in the Telecom sector, we ignore this sector altogether. We divide the companies into two equal groups such that the number of companies from a specific sector is the same in each of the two groups. Our learning task is to predict the daily returns of the second group of companies from the first group. Although this is not a prediction task that would be of most interest in practice (where we want to predict {\it future} returns), it is a good test to examine the ability of an algorithm to extract the underlying factors and existing structure in the data. 

We divide the 5500 day period into 10 intervals of 550 days. For each interval, we choose the first 400 days as the training set (with the last 100 days as validation set)  and the last 150 days as the test set. The average and standard deviation of MSE for SMFR is 128.8 (16.0), for SRRR is 129.2 (15.8), and for LASSO is 129.4 (17.1). There is minimal difference in the predictive performance of these algorithms on this dataset; however, we gain significant insight into the nature of data of by looking at the factors created by our algorithm.  Figure~\ref{fig:sp500} compares the resulting factors of SRRR and SMFR run over a period of 3000 days (this length is chosen to have clearer factors). We place companies from the same sectors next to each other in predictor and response matrices, and separate them with  green lines. The sectors, from top to bottom, are Energy, Materials, Industrials, Consumer Discretionary, Consumer Staples, Health Care, Financials, IT, and Utilities. The proposed algorithm, SMFR, captures  the sector factors to a good extent, whereas the structure of factors  identified by SRRR is less clear.

\begin{figure}[!ht]
\begin{center}
\subfigure[SMFR factors]{%
\includegraphics[scale=0.2]{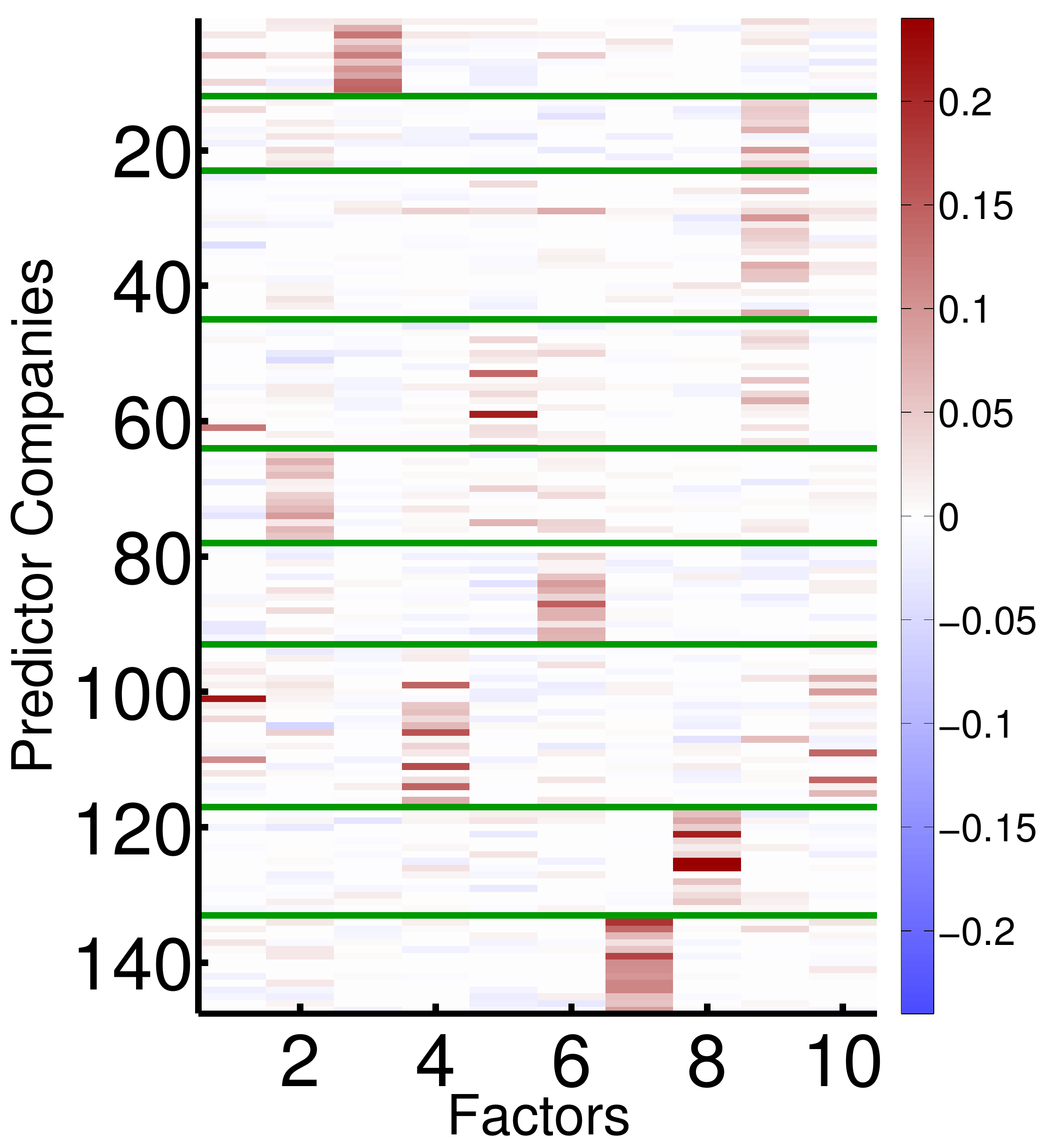}
}
\subfigure[SRRR factors]{%
\includegraphics[scale=0.2]{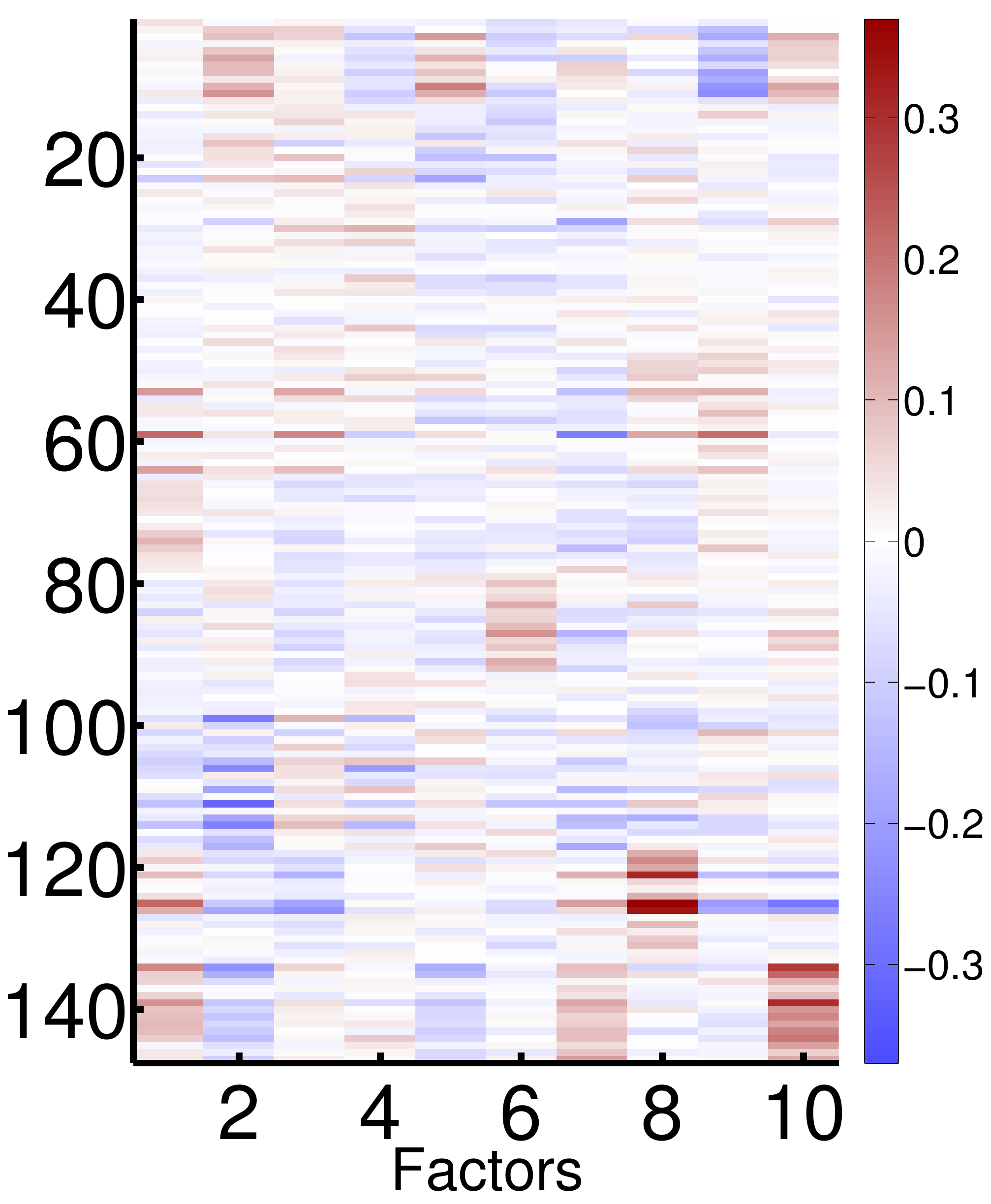}
}\caption{Factor matrices for SMFR and SRRR. \label{fig:sp500}}
\end{center}
\end{figure}

\subsection{Sparse PCA}
We compare our proposed fully sparse PCA with the well-known SPCA of Zou et al. introduced in~\citep{zou06}. We use our BIXI dataset again. Remember that there are 800 features in this dataset.  We use the first 200 data points of the dataset to simulate a high-dimensional setting. Thus, our data matrix, $\bX$, is $200 \times 800$. We compute the first 6 principal components and compare them using two metrics. First, we compute the {\em adjusted} explained variance; since the sparse principal components are not uncorrelated as in regular PCA, computing their explained variance separately is not correct. Before computing the explained variance of the $k$'th principal component, regression projection is used to remove its linear dependence to components $1$ to $k-1$. See~\citep{Zou06s} for more details on how to compute the  adjusted explained variance.

We also compare the loading sparsity. To have sparse components, we set the regularization parameters such that each component receives contributions from at most $10\%$ of the features. As a benchmark, we also consider the simple thresholding where the values of the regular principal components with absolute value smaller than a threshold are set to zero (here, we keep the top 80).

The results are summarized in Table~\ref{tab:pca}. Compared to SPCA our algorithm explains more variance in the data (a total of $36.5\%$ over the first 6 components compared to $18.6\%$) with sparser components. Also, we achieve a higher adjusted total variance compared to the simple thresholding ($36.5\%$ compared to $28.3\%$).

\newcolumntype{C}{>{\centering\arraybackslash}p{3em}}
\begin{table}[h!]
\centering
\setlength\tabcolsep{3.5pt}
\begin{tabular}{|c|c|c|c|C|C|}
\hline 
 & \multicolumn{3}{c|}{Adjusted Var (\%)} & \multicolumn{2}{c|}{\spcell{loading sparsity \\ $(\|\bA\|_{1,1})$}}  \\ 
\hline 
PC & SMFR & SPCA & Thresholding& SMFR  & SPCA  \\ 
\hline 
1 & 12.5 & 8.7& 15.2& 78& 80 \\
\hline
2 & 7.9& 2.9& 4.2&45& 79 \\
\hline
3 & 6.9& 2.0& 2.7&70& 80 \\
\hline
4 & 4.0& 1.9& 2.3&62& 78 \\
\hline
5 & 3.2& 1.7&2.1& 58 & 77 \\
\hline
6 & 2.0& 1.4&1.8&35 & 79 \\
\hline
\end{tabular} 
\vspace*{0.1 in}
\caption{Adjusted explained variance and loading sparsity.\label{tab:pca}}
\end{table}

\section{Conclusion}
We introduced a new sparse multivariate regression algorithm which imposes a low-dimensional structure  on the coefficient matrix by first decomposing it into the product of a long factor matrix and a wide loading matrix, with an elastic net penalty on the former and an $\ell_1$ penalty on the latter. We also provided a formulation to infer the number of latent factors in a more effective way than current techniques. Although the problem formulation leads to a non-convex optimization problem, we showed convergence and optimality for an alternating minimization scheme (with three sets of updates). Through experiments on simulated and real datasets, we demonstrated that the proposed algorithm is able to exploit the existing structure in the data to improve predictive performance and model selection.

\appendix
\setcounter{myprop}{0}
\setcounter{mytheo}{2}
\section{Proof of Proposition 1}

\begin{proof}
For any given $\bA$ and $\bB$, we have $f(\bA, \bB)~\geq~0$. In both minimization steps of Algorithm 1 (i.e., problems~(\ref{eq:opt1}) and~(\ref{eq:opt2})), the value of function $f$ is being decreased. Since $f$ is bounded from below, the sequence of $f(\bA_i, \bB_i)$ converges to a limit value $f^*\in\mathbb{R}$. 
\end{proof}

\section{Proof of Theorem 3}
\label{sec:appThm1}
\subsubsection*{Proof of part (i)}~\\
For problem~(\ref{eq:opt1}), decomposing $\bB$ into its columns, we can rewrite the minimization as follows:
\begin{equation} \label{eq:JCR22}
\widehat{\bB} = \argmin_{\bB} \  \sum_{j=1}^q \left\{ \frac{1}{2} \| \bY^{(j)}-\bX\bA_i\bB^{(j)}\|_2^2 + \lambda_2 \|\bB^{(j)}\|_1 \right\},
\end{equation}
where for any matrix, the superscript of $(j)$ denotes its $j$'th column. Therefore, problem~(\ref{eq:opt1}) is equivalent to $q$ separate Lasso problems, one for each response. 
\begin{mydef}
\label{def:generalposition}
A matrix $\bX_{n \times p}$ has its columns in {\it general position} if for any $k<n$, $\bX_j \notin \text{aff}\{\bX_{i_1}, \ldots, \bX_{i_{k+1}}\}, \forall j \notin \{i_1, \ldots, i_{k+1}\}$, where $\bX_i$ denotes the $i$'th column of $\bX$ and `aff' denotes the affine span. 
\end{mydef}
In~\citep{Tib13},  Tibshirani shows that:
\begin{mylemma}[\citep{Tib13}, Lemmas 3 and 4]\label{Lem:tib} Assume that we have the following Lasso problem: \[\min_b \|\by-\bH\bb\|_2^2 + \lambda\|\bb\|_1.\] If the columns of $\bH$ are in general position, then for any $\by$ and $\lambda$, the Lasso solution is unique with probability one. Moreover, if the entries of $\bH \in \mathbb{R}^{n \times m}$ are drawn from a continuous probability distribution on $\mathbb{R}^{nm}$, then its columns are in general position and thus, for any $\by$ and $\lambda$, the Lasso solution is unique with probability one.
\end{mylemma}
In problem~(\ref{eq:opt1}), we have $\bH=\bX\bA$. If the elements of $\bX$ are drawn from a continuous distribution, then its columns are in general position with probability one. In the following Lemma, we show that multiplying $\bX$ by $\bA$ with full column rank does not change this property and thus, the columns of $\bH$ are also in general position. Therefore, according to Lemma~\ref{Lem:tib}, the solution of~(\ref{eq:opt1}) is unique with probability one.
\begin{mylemma}
If the columns of $\bX_{n\times p}$ are in general position with probability one, and $\bA_{p \times m}$ has full column rank, then the columns of $\bX\bA$ are also in general position with probability one.
\end{mylemma}
\begin{proof}
Assume that for $\{i_1, \ldots, i_{k+1}\}$ and a $j$ not in that set, we have $\bX\bA_j \in \text{aff}\{\bX\bA_{i_1}, \ldots, \bX\bA_{i_{k+1}}\}$. Then, for some $\alpha_l, l=1,\ldots,k+1$, we have:
\begin{equation} \label{eq:genpos}
\bX\left(\bA_j + \sum_{l=1}^{k+1} \alpha_l\bA_{i_{l}}\right) = {\bf 0}.
\end{equation}
Since $\bX$ has its columns in general position, (\ref{eq:genpos}) holds with a non-zero probability iff $\bA_j + \sum_{l=1}^{k+1} \alpha_l\bA_{i_{l}} = {\bf 0}$, which is not possible because $\bA$ has full column rank.
\end{proof}

\subsubsection*{Proof of part (ii)}~\\
The objective function is strongly convex in $\bA$, so if there is a solution, it will be unique.

\section{Proof of Theorem 4}

Some of the proofs in this subsection exploit the biconvexity of the problem we are addressing and are based on the proofs of similar results in~\citep{Gor07}. \\

\subsubsection*{Proof of part (i)}
\begin{mydef} $\mathcal{A}$ is called the {\it algorithmic map} of Algorithm 1, if for $\bC_1 = (\bA_1, \bB_1) $ and $\bC_2 = (\bA_2, \bB_2)$ we have:
\begin{align}
\bC_2 \in \mathcal{A}(\bC_1) \quad \text{iff} & \quad f(\bA_1, \bB_2) \leq f(\bA_1, \bB), \forall \bB\in\mathbb{R}^{m\times q} \nonumber \\ \text{ and} &  \quad f(\bA_2, \bB_2) \leq f(\bA, \bB_2), \forall \bA\in\mathbb{R}^{n\times m}. \nonumber
\end{align}
In other words, $\bC_2 \in \mathcal{A}(\bC_1)$ iff we can go from $\bC_1$ to $\bC_2$ in one iteration of Algorithm~1.
\end{mydef}
\begin{mylemma} \label{Lem:closed} The algorithmic map $\mathcal{A}$ is closed, i.e., we have:
\begin{equation}
\left.\begin{tabular}{l}
$\bC_i = (\bA_i, \bB_i)$ \ and $\lim_{i\rightarrow \infty} \bC_i = \bC^*$\!\\
$\bC'_i \in \mathcal{A}(\bC_i)$ \ \quad and $\lim_{i\rightarrow \infty} \bC'_i = \bC'\!$
\end{tabular}\right\}
\!\!\Rightarrow \!\bC' \!\in\! \mathcal{A}(\bC^*)
\end{equation}
\end{mylemma}
\begin{proof}
\begin{eqnarray*}
\bC'_i \in \mathcal{A}(\bC_i)  \Rightarrow  f(\bA_i, \bB'_i) \ \leq & \!f(\bA_i,\bB), \ \forall \bB\in\mathbb{R}^{m\times q} \\  \text{and }  f(\bA'_i, \bB'_i) \ \leq &\! f(\bA, \bB'_i), \ \forall \bA\in\mathbb{R}^{n\times m}
\end{eqnarray*}
Since $f$ is continuous, we have:
\begin{align*}
f(\bA^*, \bB') = \lim_{i \rightarrow \infty} f(\bA_i,\bB_i')  & \leq  \lim_{i \rightarrow \infty} f(\bA_i,\bB) \\  & =  f(\bA^*, \bB) \quad \forall \bB\in\mathbb{R}^{m\times q} \\
f(\bA', \bB') = \lim_{i \rightarrow \infty} f(\bA'_i,\bB_i') & \leq  \lim_{i \rightarrow \infty} f(\bA,\bB'_i) \\  & =  f(\bA, \bB') \quad \forall \bA\in\mathbb{R}^{n\times m}
\end{align*}
Thus, $\bC' \in \mathcal{A}(\bC^*)$, and $\mathcal{A}$ is closed. 
 \end{proof}

 \begin{mylemma}\label{Lem:bounded} For a given starting point, $(\bA_0, \bB_0)$, the solutions $\{(\bA_i, \bB_i)\}_{i \in \mathbb{N}}$ stay in a bounded set. 
\end{mylemma}
\begin{proof}
We have
\begin{equation}
0 \leq \lambda_1\|\bA_i\|_{1,1} + \lambda_2\|\bB_i\|_{1,1} + \lambda_3\|\bA\|_F^2 \leq f(\bA_i,\bB_i) \leq f(\bA_0, \bB_0).
\end{equation}
Thus, $\{(\bA_i, \bB_i)\}_{i \in \mathbb{N}}$ stay in a bounded set.
\end{proof}

From Lemmas \ref{Lem:closed} and \ref{Lem:bounded}, we conclude that for a given starting point, the sequence of solutions $\{(\bA_i, \bB_i)\}_{i \in \mathbb{N}}$ stays in a bounded, closed, and hence compact set and thus has at least one accumulation point. \\

\subsubsection*{Proof of part (ii)}~\\
Since $\bC_{i+1} \in \mathcal{A}(\bC_{i})$, we have:
\begin{align*}
 f(\bA_{i},\bB_{i+1}) & \leq f(\bA_{i},\bB ), \ \forall \bB\in\mathbb{R}^{m\times q} \\  \text{and} \quad  f(\bA_{i+1},\bB_{i+1}) & \leq f(\bA,\bB_{i+1} ), \ \forall  \bA\in\mathbb{R}^{n\times m}
\end{align*}
Moreover, if we have $f(\bC_{i+1}) = f(\bC_{i})$, Then:
\begin{equation}
f(\bA_{i+1},\bB_{i+1}) = f(\bA_{i},\bB_{i+1}) = f(\bA_{i},\bB_{i}).
\end{equation}
Therefore, if the solution of~(\ref{eq:opt1}) is unique (i.e., $\bB_{i+1} = \bB_{i}$), then $\bC_i$ is a partial optimum, and if    the solution of~(\ref{eq:opt2}) is unique (i.e., $\bA_{i+1} = \bA_{i}$), then $\bC_{i+1}$ is a partial optimum (the latter always hold because the solution of ~(\ref{eq:opt2}) is unique).

We know that the sequence $\{\bC_i\}_{i\in\mathbb{N}}$  has at least one accumulation point, say $\bC^*$. Thus we have a convergent subsequence $\{\bC_i\}_{i\in\mathbb{K}}$ with $\mathbb{K} \subset \mathbb{N}$ that converges to $\bC^*$. Similarly, $\{\bC_{i+1}\}_{i\in\mathbb{K}}$ has an accumulation point, say $\bC^+$, to which a subsequence $\{\bC_{i+1}\}_{i\in\mathbb{L}}$ with $\mathbb{L} \subset \mathbb{K}$ converges. From Lemma~\ref{Lem:closed} we get $\bC^+\in\mathcal{A}(\bC^*)$, and using Proposition 1 we conclude $f(\bC^+) = f(\bC^*)$. Similarly, if $\bC^-$ shows the accumulation point of $\{\bC_{i-1}\}_{i\in\mathbb{K}}$, we can show $f(\bC^-) = f(\bC^*)$.

Combining the results of these two paragraphs, if the solution of~(\ref{eq:opt1}) is unique, $f(\bC^+) = f(\bC^*)$ implies that $\bC^*$ is partial optimum, and if the solution of~(\ref{eq:opt2}) is unique, $f(\bC^-) = f(\bC^*)$ implies that $\bC^*$ is partial optimum. Therefore, solution uniqueness of either~(\ref{eq:opt1}) or~(\ref{eq:opt2}) implies that $\bC^*$ is a partial optimum.

\subsection*{Proof of part (iii)}
We prove by contradiction; assume that $\|\bC_{i+1} - \bC_i\|$ does not converge to zero. Then, for infinitely many $i\in \mathbb{N}$, we have $\|\bC_{i+1} - \bC_i\| > \delta$ for a $\delta>0$. Thus, denoting the accumulation points of sequences $\{\bC_i\}_{i\in\mathbb{N}}$ and $\{\bC_{i+1}\}_{i\in\mathbb{N}}$ respectively by $\bC^*$ and $\bC^+$, we must have $\|\bC^* - \bC^+\| > \delta$ and hence $\bC^+ \neq \bC^*$. On the other hand, since $\bC^*$ is a partial optimum and $\bC^+\in\mathcal{A}(\bC^*)$, we have:
\begin{equation}
f(\bA^*,\bB^*) = f(\bA^*, \bB^+) = f(\bA^+,\bB^+).
\end{equation}
Both $\bA^*$ and $\bB^*$ are full rank and thus, by Theorem 1, $\bB^+ = \bB^*$, $\bA^+=\bA^*$, and consequently, $\bC^+ = \bC^*$. This is in contradiction with the result of the previous paragraph and hence, $\|\bC_{i+1} - \bC_i\|$ converges to 0.

\section*{Acknowledgement}
This work is supported by the Natural Sciences and Engineering Research Council of Canada (NSERC).
\section*{References}
\bibliographystyle{chicago}
\bibliography{SMFR}

\end{document}